\documentclass[10pt]{article} % For LaTeX2e
\usepackage[preprint]{tmlr}

\usepackage{amsmath,amsfonts,bm}

\def\eqref#1{equation~\ref{#1}}
\def\1{\bm{1}}

\DeclareMathAlphabet{\mathsfit}{\encodingdefault}{\sfdefault}{m}{sl}
\SetMathAlphabet{\mathsfit}{bold}{\encodingdefault}{\sfdefault}{bx}{n}

\RequirePackage{algorithm}
\RequirePackage{algorithmic}
\usepackage{hyperref}
\usepackage{url}

\usepackage{microtype}
\usepackage{graphicx}
\usepackage{subcaption}
\usepackage{booktabs} % for professional tables
\usepackage{multirow}
\usepackage{hyperref}
\usepackage{comment}
\usepackage{enumitem}

\usepackage{amsmath}
\usepackage{amssymb}
\usepackage{mathtools}
\usepackage{amsthm}
\usepackage{algorithm}
\usepackage{algorithmic}
\usepackage[capitalize,noabbrev]{cleveref}

\theoremstyle{plain}
\newtheorem{theorem}{Theorem}[section]
\newtheorem{proposition}[theorem]{Proposition}

\theoremstyle{definition}

\theoremstyle{remark}

\usepackage[textsize=tiny]{todonotes}

\title{Collaborative Synthetic Data Generation for Knowledge Transfer in  Federated Learning}

\author{\name Maximilian Andreas Hoefler \email maximilian.andreas.hoefler@hhi.fraunhofer.de \\
      \addr Fraunhofer Heinrich Hertz Institute\\
      \AND
      \name Karsten Mueller \email karsten.mueller@hhi.fraunhofer.de \\
      \addr Fraunhofer Heinrich Hertz Institute
      \AND
      \name Wojciech Samek \email wojciech.samek@hhi.fraunhofer.de\\
      \addr Fraunhofer Heinrich Hertz Institute\\
      BIFOLD \\ Technical University Berlin 
      }

\def\month{MM}  % Insert correct month for camera-ready version
\def\year{YYYY} % Insert correct year for camera-ready version
\def\openreview{\url{https://openreview.net/forum?id=XXXX}} % Insert correct link to OpenReview for camera-ready version

\begin{document}

\maketitle

\begin{abstract}
One-shot federated learning (OSFL) addresses the communication overhead of 
federated learning by limiting training to a single round, but doing so 
without sacrificing model quality is non-trivial, particularly when client data distributions diverge. Recent work has addressed this challenge by aggregating client knowledge on the 
server through the construction of transferable synthetic datasets or 
distillates. However, most of these methods lack formal privacy guarantees, leaving a gap in jointly achieving low communication, 
robustness to heterogeneity, and rigorous privacy. We propose 
FedKT-CSD (Federated Knowledge Transfer via Collaborative Synthetic Data), 
a framework inspired by neural image compression that closes this gap by 
leveraging publicly pretrained autoencoders as a shared latent space. Each 
client encodes its private data in a single forward pass, computes 
class-conditional latent statistics, and transmits these to the server. The 
server aggregates these statistics via secure aggregation, adds calibrated 
differential privacy noise, and decodes a synthetic dataset for training a 
global model and further downstream tasks. This design provides formal 
$(\varepsilon,\delta)$-differential privacy by construction, while keeping 
client-side computation and communication lightweight. Despite operating 
under privacy constraints, FedKT-CSD is competitive with and even 
outperforms non-private baselines across diverse datasets and heterogeneity 
settings, and scales to a large number of clients. Our code is available 
at: \url{https://github.com/an7123/FedKT-CSD}
\end{abstract}

\section{Introduction}
\label{sec:introduction}

Federated Learning (FL)~\cite{FedAvg} has emerged as a powerful paradigm for training models collaboratively across decentralized data, under the promise of privacy preservation. Nonetheless, the presence of heterogeneous data across clients poses a significant challenge, often leading to slow convergence and suboptimal model performance. A wide range of approaches have been proposed to address this, from personalized federated learning (pFL)~\cite{APPLE,FedFOMO,pFedMe,PartialFed,FedPHP,FedRep} to data and representation sharing strategies~\cite{data_sharing_FL,FedGen,FedFTG,CCVR,FedFed, fedxds}. Despite their effectiveness, these methods typically require many communication rounds which can be problematic for real-world cross-device deployment. 

A particularly attractive direction for communication constrained applications is one-shot federated learning (OSFL), where clients communicate with the server exactly once. Recent OSFL methods leverage the aforementioned data sharing principle, showing the efficacy in improving one-shot performance. Methods such as FedD3~\cite{FedD3}, FedSD2C~\cite{FedSD2C}, FedCVAE~\cite{one-shot-cvae}, DENSE~\cite{Dense}, CoBoosting~\cite{coboosting} and \cite{fedpft} approach OSFL by having each client train a local generative model or distill its data, then transmit the result to the server, which aggregates the contributions into a shared synthetic dataset or ensemble. These approaches thus eliminate repeated parameter exchanges, naturally support heterogeneous architectures, and can produce datasets that can be reused for pretraining, distillation, or personalization.

However, existing OSFL methods face a fundamental limitation: few provide formal privacy guarantees as a built-in mechanism with regards to preserving utility. Privacy leakage occurs because clients transmit trained model parameters, generator weights, or distilled data, where information about individual training records can be leaked. Retrofitting differential privacy (DP) \cite{dwork2014algorithmic} onto these approaches is difficult; the high dimensionality of model parameters inflates sensitivity, DP noise degrades already-fragile local models, and multi-step training complicates privacy accounting. Meanwhile, training local generators and models such as in \cite{one-shot-cvae, Dense, coboosting} can also heavily increase computational cost on clients.

This leads to a direct question: \textit{Can we devise a framework which leverages the benefits of representation sharing in an OSFL setting with formal privacy guarantees and lightweight client computation?}

We observe that modern pretrained autoencoders, originally developed for neural image compression~\cite{dc-ae, balle}, provide exactly the structure needed to answer this question. These models map images to compact latent vectors and back with high fidelity, using publicly available weights that require no adaptation. The key insight is that, given such a shared public encoder, each client can encode its private images in a single forward pass and compute simple per-class statistics in the latent space. These statistics can be aggregated across clients using secure aggregation, which we implement such that partitioning data across any number of clients produces exactly the same as if data were aggregated centrally. Moreover, the latent vectors are also low-dimensional and have bounded sensitivity, making them ideal targets for calibrated Gaussian noise under differential privacy \cite{dworkDP}.

Building on this insight, we propose \textbf{FedKT-CSD} (Federated Knowledge Transfer via Collaborative Synthetic Data). Each client receives a frozen, publicly pretrained autoencoder and encodes its private images via a single forward pass with no optimization or training on-device. The client computes class-conditional latent statistics, protects them under differential privacy, and transmits only these lightweight summaries to the server. The server aggregates the noisy statistics across all clients, recovers per-class Gaussian distributions in latent space, samples from them, and decodes the samples into a synthetic image dataset. This dataset can then be used to train a global classifier, a feature extractor for personalized FL, or auxiliary data for multi-round methods.

Our approach achieves three critical objectives simultaneously: (1) privacy by design, since the autoencoder is public and frozen and only DP-protected class-conditional statistics are transmitted; (2) communication efficiency, as the approach requires a single communication round where each client uploads only statistics rather than model parameters; and (3) computational efficiency, since clients perform only a forward pass through the encoder with no on-device training, optimization, or gradient computation.

Our contributions are:
\begin{itemize}
    \item \textbf{One-shot DP synthetic data generation.} We show that deep compression autoencoders provide a shared latent space in which class-conditional statistics can be computed with differential privacy. This enables a one-shot FL pipeline that generates synthetic data with formal $(\varepsilon,\delta)$-DP guarantees, a single communication round, and no on-device training.

    \item \textbf{Competitive with SOTA under DP.} Despite operating under DP, FedKT-CSD is competitive and outperforms existing one-shot FL baselines across  diverse datasets and heterogeneity settings. Moreover, our method is also invariant to both the degree of heterogeneity and the number of clients, properties that no existing OSFL method possesses.

    \item \textbf{Versatile downstream use.} The DP synthetic dataset is not limited to server-side training. We demonstrate its effectiveness for downstream personalized FL, where it serves as pretraining data and achieves improvements over multi-round personalized FL baselines.

    \item \textbf{Practical efficiency.} Each client uploads payloads in the regime of Kilobytes and requires seconds of computation with no GPU training. This makes FedKT-CSD deployable in bandwidth- and compute-constrained settings where existing methods are impractical.
\end{itemize}

\begin{figure}
    \centering
    \includegraphics[width=1.0\linewidth]{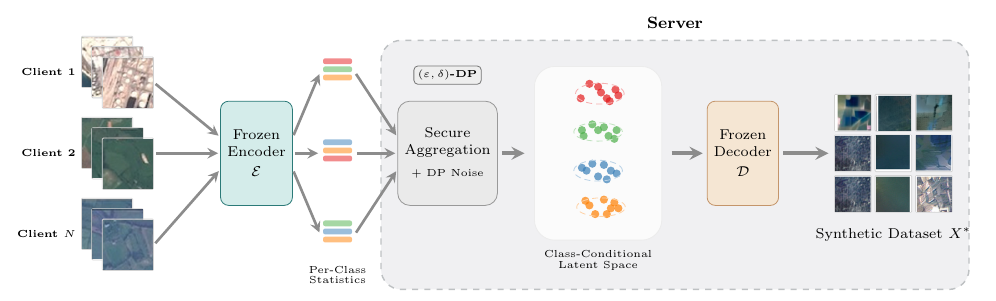}
    \caption{Overview of FedKT-CSD. Each client holds private images from a heterogeneous 
subset of classes. A shared frozen encoder maps images to compact latent vectors, from 
which per-class statistics (colored bars) are computed locally. These lightweight 
summaries are transmitted via secure aggregation and protected with calibrated DP noise 
on the server. The server recovers class-conditional Gaussian distributions in the 
latent space (colored clusters), samples new latent vectors, and decodes them into a 
synthetic dataset $X^*$ via the frozen decoder.}
    \label{fig:Overview}
\end{figure}
 
\section{Related Works and Preliminaries}

\textbf{One-Shot Federated Learning.} One-shot FL methods can be organized by how clients communicate with the server. One family transmits local models or generators, from which the server extracts knowledge via distillation. DENSE~\cite{Dense} and Co-Boosting~\cite{coboosting} follow this approach, aggregating client model ensembles through data-free knowledge distillation. Both evaluate privacy empirically through membership inference but provide no formal guarantees. A second family transmits data, or data-derived features, ranging from distilled datasets, synthetic samples, or parametric distributions. FedD3~\cite{FedD3} constructs compact client summaries via dataset distillation. FedSD2C~\cite{FedSD2C} optimizes synthetic distillates through a pretrained autoencoder to minimize information loss relative to the original data. Similarly, \cite{one-shot-cvae} shares locally trained generators to avoid model aggregation, though without privacy analysis. Most closely related to our work is FedPFT~\cite{fedpft}, which also encodes images through a frozen pretrained model and transmits class-conditional Gaussian mixtures in a single round. Our approach differs in that we aggregate additive statistics rather than per-client mixture models, enabling secure aggregation, exact partition invariance, and cleaner DP accounting. In the FedPFT privacy model, noise is added per-client to high-dimensional GMM parameters, which does not benefit from secure aggregation and is likely to degrade at meaningful privacy budgets. Moreover, we use an autoencoder rather than a discriminative feature extractor, producing a reusable pixel-space synthetic dataset rather than features tied to a specific classifier head, which is more useful for downstream tasks. We discuss the precise differences further in the appendix.

\textbf{Data and Representation Sharing.} Underpinning the recent works in OSFL are works on data sharing. To improve transferability, recent work explores sharing auxiliary information beyond parameters. FedGen~\cite{FedGen} and FedFTG~\cite{FedFTG} use knowledge distillation from client logits to train server-side generators or refine global models. CCVR~\cite{CCVR} and FedFed~\cite{FedFed} share representation statistics to calibrate classifiers or align features. In addition, \cite{fedxds} use data derived from XAI-attribution methods, and \cite{diffusion_federated} uses diffusion models to account for classifier drift. More recently, foundation models enable high-quality synthetic data generation: Pre-Text~\cite{Pre-Text}, dp-prompt~\cite{dp-prompt}, and dp-kde~\cite{dp-kde} leverage LLMs in the text domain. These approaches typically require large pretrained models and may incur significant computational or privacy costs.

\section{Methodology}
\begin{algorithm}[t]
\caption{FedKT-CSD Algorithm}
\label{alg:fedktcsd}
\textbf{Input:} Frozen encoder $\mathcal{E}$, decoder $\mathcal{D}$; $N$ clients with labeled data; privacy budget $(\varepsilon, \delta)$; clipping radius $R$; synthetic samples per class $N_c$ \\[4pt]
\textbf{Output:} DP synthetic dataset $X^*$

\vspace{4pt}
\hrule
\vspace{4pt}
\textbf{Client side} (each client $i$, in parallel, no training)
\vspace{2pt}
\begin{enumerate}[leftmargin=1.5em, itemsep=1pt, label=\arabic*.]
    \item \textbf{for} each image $(x_{i,j}, y_{i,j}) \in \mathcal{D}_i$ \textbf{do}
    \begin{enumerate}[leftmargin=1.5em, itemsep=1pt, label=\alph*.]
        \item Encode: $z_{i,j} \gets \mathcal{E}(x_{i,j})$
        \item Clip to bounded norm: $\bar{z}_{i,j} \gets \mathrm{Clip}_R(z_{i,j})$
    \end{enumerate}
    \item \textbf{for} each class $c$ present in client $i$'s data \textbf{do}
    \begin{enumerate}[leftmargin=1.5em, itemsep=1pt, label=\alph*.]
        \item Sum of latents: $\mathbf{m}_c^i \gets \sum_j \bar{z}_{i,j}$
        \item Sum of outer products: $\mathbf{M}_c^i \gets \sum_j \bar{z}_{i,j}\bar{z}_{i,j}^\top$
        \item Sample count: $n_c^i \gets$ number of class-$c$ images
    \end{enumerate}
    \item Upload $\{(\mathbf{m}_c^i, \mathbf{M}_c^i, n_c^i)\}_c$ via secure aggregation
\end{enumerate}

\vspace{2pt}
\hrule
\vspace{4pt}
\textbf{Server side} (aggregate, add noise, generate)
\vspace{2pt}
\begin{enumerate}[leftmargin=1.5em, itemsep=1pt, label=\arabic*., start=5]
    \item Aggregate across clients: $\mathbf{S}_{\mu,c} \gets \sum_i \mathbf{m}_c^i$, \; $\mathbf{S}_{\Sigma,c} \gets \sum_i \mathbf{M}_c^i$, \; $n_c \gets \sum_i n_c^i$
    \item Add calibrated Gaussian noise to each sum (Prop.~\ref{prop:dp_guarantee})
    \item \textbf{for} each class $c = 1, \dots, K$ \textbf{do}
    \begin{enumerate}[leftmargin=1.5em, itemsep=1pt, label=\alph*.]
        \item Recover noisy mean: $\boldsymbol{\mu}_c \gets$ noisy latent sum $/ \, n_c$
        \item Recover covariance: second moment $/ \, n_c$ minus outer product of mean
        \item Apply bias correction and PSD projection
        \item Sample $N_c$ latents from $\mathcal{N}(\boldsymbol{\mu}_c, \boldsymbol{\Sigma}_c)$
        \item Decode to images: $x_c^{(k)} \gets \mathcal{D}(z_c^{(k)})$
    \end{enumerate}
    \item \textbf{Return} $X^* = \{(x_c^{(k)}, c)\}$ for downstream training
\end{enumerate}
\end{algorithm}

\label{sec:method}

\subsection{Problem Setting}

We consider one-shot federated learning (OSFL). $N$ clients hold private datasets $\mathcal{D}_i = \{(x_{i,j}, y_{i,j})\}_{j=1}^{n_i}$ with $y_{i,j} \in \{1,\dots,K\}$, drawn from heterogeneous distributions. The goal is to produce a global model $f_\theta$ minimizing
\begin{equation}
\label{eq:objective}
\min_\theta \; \sum_{i=1}^N \frac{n_i}{\sum_j n_j}\,\mathcal{L}_{\mathcal{D}_i}(\theta),
\end{equation}
subject to (i)~a \textbf{single communication round} and (ii)~\textbf{record-level differential privacy}. Our key observation is that statistics of latent representations are additive: class-conditional sums computed locally combine across clients without information loss, so each client transmits only lightweight summary statistics from which the server generates a DP synthetic dataset $X^*$ (\autoref{fig:Overview}). We follow the standard conditional-generation setting where class labels are treated as public information. The privacy guarantee therefore protects the private feature content of each record, not its class membership although we outline how to achieve this in the appendix.

\subsection{Pretrained Autoencoder and Encoding}
\label{sec:autoencoder}

We use a publicly pretrained autoencoder with encoder $\mathcal{E}\colon \mathbb{R}^{H \times W \times C} \to \mathbb{R}^d$ and decoder $\mathcal{D}\colon \mathbb{R}^d \to \mathbb{R}^{H \times W \times C}$, both frozen throughout.

Each client $i$ encodes its images via $z_{i,j} = \mathcal{E}(x_{i,j})$, a single forward pass with no optimization or training on-device. After clipping each latent to $\ell_2$-norm at most $R$ (Section~\ref{sec:clipping}), the client groups its samples by class label. Let $\mathcal{I}_{i,c} = \{j : y_{i,j} = c\}$ denote the index set of client $i$'s samples belonging to class $c$. For each class it holds, the client computes three quantities:
\begin{equation}
\label{eq:client_stats}
\mathbf{m}_c^i = \sum_{j \in \mathcal{I}_{i,c}} \bar{z}_{i,j}, \qquad
\mathbf{M}_c^i = \sum_{j \in \mathcal{I}_{i,c}} \bar{z}_{i,j}\,\bar{z}_{i,j}^\top, \qquad
n_c^i = |\mathcal{I}_{i,c}|,
\end{equation}
where $\bar{z}_{i,j} = \mathrm{Clip}_R(z_{i,j})$ is the clipped latent vector. Here $\mathbf{m}_c^i \in \mathbb{R}^d$ is the sum of clipped latents for class $c$ (not the mean), $\mathbf{M}_c^i \in \mathbb{R}^{d \times d}$ is the sum of their outer products, and $n_c^i$ is the number of class-$c$ samples held by client $i$. The triple $(\mathbf{m}_c^i, \mathbf{M}_c^i, n_c^i)$ is all that leaves the client. We transmit sums rather than means because sums are additive: they can be aggregated across clients without knowing per-client sample counts, compose naturally with secure aggregation (Section~\ref{sec:secagg}), and have bounded per-record sensitivity for DP.

\subsection{Aggregation and Noise Injection}
\label{sec:secagg}

The server collects the per-client triples and aggregates them via secure aggregation, a cryptographic protocol that computes element-wise sums across clients while revealing only the totals to the server, never individual contributions (we describe secure aggregation further in the appendix). The resulting global quantities are:
\begin{equation}
\label{eq:global_sums}
\mathbf{S}_{\mu,c} = \sum_{i=1}^N \mathbf{m}_c^i \in \mathbb{R}^d, \qquad
\mathbf{S}_{\Sigma,c} = \sum_{i=1}^N \mathbf{M}_c^i \in \mathbb{R}^{d \times d}, \qquad
n_c = \sum_{i=1}^N n_c^i.
\end{equation}
Because these are simple sums, the result is identical to what would be computed if all class-$c$ data resided on a single machine; the partitioning across clients has no effect. Gaussian noise calibrated to the sensitivity bounds (Proposition~\ref{prop:dp_guarantee}) is then added to the aggregated sums:
\begin{equation}
\label{eq:noisy_sums}
\widetilde{\mathbf{S}}_{\mu,c} = \mathbf{S}_{\mu,c} + \eta_\mu, \qquad
\widetilde{\mathbf{S}}_{\Sigma,c} = \mathbf{S}_{\Sigma,c} + \eta_\Sigma,
\end{equation}
where $\eta_\mu \sim \mathcal{N}(0,\sigma_\mu^2 I_d)$ and $\eta_\Sigma \sim \mathcal{N}(0,\sigma_\Sigma^2 I_{d \times d})$. After this step, all subsequent operations (normalization, sampling, decoding, model training) are deterministic functions of the noisy sums and incur no additional privacy cost.

\subsection{Post-Processing}
\label{sec:post_processing}

Given the noisy sums $(\widetilde{\mathbf{S}}_{\mu,c}, \widetilde{\mathbf{S}}_{\Sigma,c}, n_c)$, the server recovers class-conditional statistics. The DP-protected mean $\boldsymbol{\mu}_c^{\mathrm{DP}} \in \mathbb{R}^d$ and raw covariance estimate $\widehat{\boldsymbol{\Sigma}}_c \in \mathbb{R}^{d \times d}$ are:
\begin{equation}
\label{eq:dp_mean}
\boldsymbol{\mu}_c^{\mathrm{DP}} = \frac{\widetilde{\mathbf{S}}_{\mu,c}}{n_c}, \qquad
\widehat{\boldsymbol{\Sigma}}_c 
  = \frac{\widetilde{\mathbf{S}}_{\Sigma,c}}{n_c} 
    - \boldsymbol{\mu}_c^{\mathrm{DP}}\,(\boldsymbol{\mu}_c^{\mathrm{DP}})^\top.
\end{equation}
The first term $\widetilde{\mathbf{S}}_{\Sigma,c}/n_c$ is the noisy second moment; subtracting the outer product of the mean yields the covariance. Two corrections are needed before $\widehat{\boldsymbol{\Sigma}}_c$ can be used for sampling:

\paragraph{Bias correction.} Because $\boldsymbol{\mu}_c^{\mathrm{DP}}$ is itself noisy, the subtracted outer product is biased upward: $\mathbb{E}[\boldsymbol{\mu}_c^{\mathrm{DP}}(\boldsymbol{\mu}_c^{\mathrm{DP}})^\top] = \boldsymbol{\mu}_c\boldsymbol{\mu}_c^\top + (\sigma_\mu/n_c)^2 I_d$, where the second term is the variance of the noise in the mean. This increases the subtracted quantity, \emph{decreasing} the covariance estimate along the diagonal. We correct analytically:
\begin{equation}
\label{eq:bias_correction}
\widehat{\boldsymbol{\Sigma}}_c \;\leftarrow\;
  \widehat{\boldsymbol{\Sigma}}_c + \frac{\sigma_\mu^2}{n_c^2}\,I_d.
\end{equation}

\paragraph{PSD projection.} The additive noise on $\widetilde{\mathbf{S}}_{\Sigma,c}$ can result in $\widehat{\boldsymbol{\Sigma}}_c$ being indefinite (some eigenvalues negative), which would prevent sampling. We project onto the positive semidefinite cone and compute the eigendecomposition $\widehat{\boldsymbol{\Sigma}}_c = Q\Lambda Q^\top$ (with $Q$ orthogonal and $\Lambda$ diagonal) and clamp all eigenvalues to a minimum threshold $\tau > 0$:
\begin{equation}
\label{eq:psd_projection}
\boldsymbol{\Sigma}_c^{\mathrm{DP}} = Q\,\max(\Lambda,\,\tau I)\,Q^\top,
\end{equation}
where $\tau$ is small (e.g., $10^{-6}$). The result $\boldsymbol{\Sigma}_c^{\mathrm{DP}}$ is the final DP-protected covariance used for generation. Both corrections are deterministic post-processing of the noisy sums and incur no additional privacy cost.

\subsection{Synthetic Data Generation and Downstream Use}
\label{sec:synth}

The server now possesses, for each class $c$, a DP-protected mean $\boldsymbol{\mu}_c^{\mathrm{DP}} \in \mathbb{R}^d$ and covariance $\boldsymbol{\Sigma}_c^{\mathrm{DP}} \in \mathbb{R}^{d \times d}$ that together parameterize a class-conditional Gaussian in the autoencoder's latent space. For each class $c$, $N_c$ latent vectors are sampled from this Gaussian and decoded back to pixel space via the frozen decoder $\mathcal{D}$:
\begin{equation}
\label{eq:sampling}
z_c^{(k)} \sim \mathcal{N}\!\big(\boldsymbol{\mu}_c^{\mathrm{DP}},\,\boldsymbol{\Sigma}_c^{\mathrm{DP}}\big), \qquad
x_c^{(k)} = \mathcal{D}(z_c^{(k)}), \qquad k = 1,\dots,N_c.
\end{equation}
The synthetic dataset $X^* = \bigcup_{c=1}^K \{(x_c^{(k)}, c)\}_{k=1}^{N_c}$ aggregates decoded images across all $K$ classes. Since $X^*$ is derived entirely from the DP-protected statistics, it can be generated in any quantity, producing more samples is just additional sampling and decoding, at no additional privacy cost.

In our experiments we train a neural network on $X^*$ via standard ERM:
\begin{equation}
\label{eq:downstream}
\theta^* = \arg\min_\theta \; \frac{1}{|X^*|} \sum_{(x,y) \in X^*} \ell\!\big(f_\theta(x),\, y\big).
\end{equation}
$X^*$ is equally suitable as pretraining data for multi-round FL, a distillation corpus, or input to any other learning pipeline. We refer do \autoref{fig:Overview} for a visualization of our method.

% =============================================================================
\section{Privacy Analysis}
\label{sec:privacy}

The DP analysis is straightforward compared to iterative FL methods. Statistics are released exactly once, sensitivity is determined by a public clipping threshold, and all downstream operations are post-processing. This DP framework does not protect class membership and we assume the label space is public knowledge. We show how to privatize class membership in the supplementary materials.

\subsection{Clipping}
\label{sec:clipping}

Before computing statistics, each client clips its encoded latents to $\ell_2$-norm to $R$:
\begin{equation}
\label{eq:clipping}
\bar{z}_{i,j} = z_{i,j} \cdot \min\!\Big(1,\;\frac{R}{\|z_{i,j}\|_2}\Big).
\end{equation}
Clipping ensures that each individual record's contribution to the sums $\mathbf{S}_{\mu,c}$ and $\mathbf{S}_{\Sigma,c}$ is bounded, which is necessary for calibrating the DP noise. The threshold $R$ is set data-independently from the autoencoder prior: under $p(z) = \mathcal{N}(0, I_d)$, the norm $\|z\|_2$ concentrates near $\sqrt{d}$, so $R = \alpha\sqrt{d}$ with a moderate $\alpha$ (e.g., $\alpha = 3$) covers the vast majority of encoded latents without inspecting any private data.

\subsection{Sensitivity}
\label{sec:sensitivity}

Under the public-label assumption above, we adopt \emph{within-class replacement} adjacency. Adjacent datasets $\mathcal{D} \sim \mathcal{D}'$ differ by replacing one class-$c$ record with another class-$c$ record, leaving all counts $n_c$ unchanged. Under this definition, replacing a clipped latent $\bar{z}$ (with $\|\bar{z}\|_2 \le R$) by $\bar{z}'$ (with $\|\bar{z}'\|_2 \le R$) changes the sum $\mathbf{S}_{\mu,c}$ by $\bar{z}' - \bar{z}$ and the second-moment sum $\mathbf{S}_{\Sigma,c}$ by $\bar{z}'(\bar{z}')^\top - \bar{z}\bar{z}^\top$. The $\ell_2$-sensitivities are therefore:
\begin{equation}
\label{eq:sensitivity}
\Delta_\mu = \sup_{\mathcal{D} \sim \mathcal{D}'} \|\mathbf{S}_{\mu,c}(\mathcal{D}) - \mathbf{S}_{\mu,c}(\mathcal{D}')\|_2 \le 2R, \qquad
\Delta_\Sigma = \sup_{\mathcal{D} \sim \mathcal{D}'} \|\mathbf{S}_{\Sigma,c}(\mathcal{D}) - \mathbf{S}_{\Sigma,c}(\mathcal{D}')\|_F \le 2R^2,
\end{equation}
by the triangle inequality and the norm bound on clipped latents.

\subsection{Privacy Guarantee}

The mechanism results in two quantities per class: the noisy mean sum $\widetilde{\mathbf{S}}_{\mu,c}$ and the noisy second-moment sum $\widetilde{\mathbf{S}}_{\Sigma,c}$, each protected by the Gaussian mechanism. We split the total privacy budget $(\varepsilon, \delta)$ equally between them. We use the \emph{analytic Gaussian mechanism}~\cite{analytic_gaussian} rather than the textbook closed-form calibration because it provides a valid and tighter calibration for all $\varepsilon > 0$, including our experimental regime.

For a query $q(\mathcal{D})$ with $\ell_2$-sensitivity $\Delta$, we release
\begin{equation}
\label{eq:gaussian_release}
\widetilde{q}(\mathcal{D}) = q(\mathcal{D}) + \xi, \qquad \xi \sim \mathcal{N}(0, \sigma^2 I),
\end{equation}
where $\sigma$ is calibrated by the analytic Gaussian mechanism. In our case, this gives
\begin{equation}
\widetilde{\mathbf{S}}_{\mu,c} = \mathbf{S}_{\mu,c} + \eta_{\mu,c}, \qquad
\widetilde{\mathbf{S}}_{\Sigma,c} = \mathbf{S}_{\Sigma,c} + \eta_{\Sigma,c},
\end{equation}
where $\eta_{\mu,c} \sim \mathcal{N}(0,\sigma_\mu^2 I_d)$ and, viewing $\mathbf{S}_{\Sigma,c} \in \mathbb{R}^{d \times d}$ as a vector in $\mathbb{R}^{d^2}$, $\eta_{\Sigma,c}$ is isotropic Gaussian noise in that ambient space with variance $\sigma_\Sigma^2$ per coordinate.

\begin{proposition}[FedKT-CSD Differential Privacy]
\label{prop:dp_guarantee}
Under the notation above, let $\varepsilon' = \varepsilon/2$ and $\delta' = \delta/2$. Set the noise scales $\sigma_\mu$ and $\sigma_\Sigma$ to be the smallest positive values such that
\begin{equation}
\label{eq:agm_calibration}
\Phi\!\left(\frac{\Delta}{2\sigma} - \frac{\varepsilon' \sigma}{\Delta}\right)
-
e^{\varepsilon'}
\Phi\!\left(-\frac{\Delta}{2\sigma} - \frac{\varepsilon' \sigma}{\Delta}\right)
\le \delta',
\end{equation}
where $\Phi$ is the standard normal CDF, with $\Delta=\Delta_\mu$ for $\sigma_\mu$ and $\Delta=\Delta_\Sigma$ for $\sigma_\Sigma$. Then, for a fixed class $c$, the joint disclosure of $(\widetilde{\mathbf{S}}_{\mu,c},\,\widetilde{\mathbf{S}}_{\Sigma,c})$ satisfies $(\varepsilon,\delta)$-differential privacy at the record level.
\end{proposition}

\begin{proof}
By the analytic Gaussian mechanism~\cite{analytic_gaussian}, a query with $\ell_2$-sensitivity $\Delta$ satisfies $(\varepsilon',\delta')$-DP when Gaussian noise with standard deviation $\sigma$ is calibrated according to Eq.~\eqref{eq:agm_calibration}. Applying this separately to the two queries with sensitivities $\Delta_\mu = 2R$ and $\Delta_\Sigma = 2R^2$ yields $(\varepsilon/2,\delta/2)$-DP for each release. By basic composition, the joint release satisfies $(\varepsilon/2{+}\varepsilon/2,\;\delta/2{+}\delta/2) = (\varepsilon,\delta)$-DP.
\end{proof}

All subsequent operations such as normalization (Eq.~\ref{eq:dp_mean}), bias correction (Eq.~\ref{eq:bias_correction}), PSD projection (Eq.~\ref{eq:psd_projection}), sampling, decoding, and model training are deterministic functions of the released noisy sums and preserve $(\varepsilon,\delta)$-DP by the post-processing theorem~\citep{dwork2014algorithmic}. The effective noise in the mean scales as $\sigma_\mu/n_c$, so larger per-class counts directly improve utility. Within-class replacement protects the \emph{values} of individual records (i.e., which specific image a client contributes to class $c$) but does not protect class membership. The alternative add/remove adjacency, which does hide class membership, would reduce sensitivity to $R$ (resp.\ $R^2$) since only one term is added or removed rather than swapped, but requires additional care in handling the variable count $n_c$. However, class membership can be protected via secure aggregation and shuffling methods; see the appendix and \cite{class_shuffling}. In addition, we discuss making the class count DP in the appendix.

\section{Experimental Results}
\label{sec:results}

\textbf{Datasets.}
We evaluate on four image classification benchmarks chosen to cover diverse visual domains: ImageNette (10 classes, 9.5k ImageNet subset) and CIFAR-100 (100 classes, 50k natural images) as standard benchmarks for image classification, and EuroSAT (10 classes, 27k satellite images), and BloodMNIST (8 classes, 12k microscopy images) since these constitute real-world and out-of-distribution datasets for the pretrained VAE encoder-decoder pair. This selection tests whether the pretrained autoencoder, originally trained on natural images, generalizes to other domains. 

\textbf{Heterogeneity simulation.}
We follow \cite{pflib} for experimental setups. The \emph{practical} setting draws per-client label distributions from a Dirichlet distribution with concentration parameter $\alpha$ where smaller $\alpha$ means more heterogeneity. The \emph{pathological} setting assigns each client a disjoint subset of classes (2 classes per client for 10- and 8-class datasets; 10 for CIFAR-100). Together these cover the spectrum from moderate label imbalance to complete class separation.

\textbf{Implementation.}
We follow \cite{pflib} in their implementation. Default client number is 20 clients with full participation. The autoencoder is DC-AE f32c32~\citep{dc-ae}. Privacy parameters are $\varepsilon{=}10$, $\delta{=}10^{-5}$ unless stated otherwise. Downstream classifier is ResNet-18 trained on synthetic data only and evaluated on real test data of the corresponding dataset. All OSFL baselines use a ResNet-18 unless otherwise specified.

% ─────────────────────────────────────────────────────────────────
\subsection{One-Shot FL Performance Comparison}
\label{sec:osfl_comparison}

Tables~\ref{tab:osfl_20c} compare FedKT-CSD against five one-shot FL baselines (FedSD2C, DENSE, CoBoosting, FedD3, and FedCVAE) across four datasets and three heterogeneity settings. All baselines operate without differential privacy, while FedKT-CSD satisfies $(\varepsilon{=}10, \delta{=}10^{-5})$-DP. Despite this stricter threat model, FedKT-CSD outperforms every baseline on every dataset and heterogeneity setting.

The reason lies in how each method handles the heterogeneity-communication tradeoff. The baselines train local models (generators, classifiers, or distilled representations) on each client's skewed data, then aggregate these models on the server. When data is heterogeneous, local models are biased toward the client's label subset, and single-round aggregation cannot fully correct this bias. This effect is visible in the tables. Every baseline degrades by 20--40\% between $\alpha{=}0.1$ and the pathological setting, because pathological splits produce heavily biased local models.

FedKT-CSD sidesteps this entirely. Because we aggregate per-class statistics rather than model parameters, the server reconstructs the same global class-conditional distributions regardless of how labels are distributed across clients. The three heterogeneity columns within each dataset are identical up to DP noise, a property that no model-aggregation approach can match. This heterogeneity invariance is a structural guarantee of the additive aggregation scheme.

Furthermore, the performance gap is widest on datasets where baseline local models are most fragile. On CIFAR-100, where pathological splits give each client only 10 of 100 classes, baselines collapse below 21\% while FedKT-CSD maintains $\sim$29\%. On EuroSAT and BloodMNIST, which have fewer classes and more samples per class, baselines fare better but FedKT-CSD still leads by 7--15 points. The pattern is consistent: the harder the heterogeneity makes local model training, the larger our advantage.

\begin{table*}[t]
\centering
\caption{Comparison of OSFL methods. Path. = pathological split (2 classes/client for 10- and 8-class datasets; 10 classes/client for CIFAR-100). Best in \textbf{bold}.}
\label{tab:osfl_20c}
\resizebox{\textwidth}{!}{%
\begin{tabular}{ll ccc ccc ccc ccc}
\toprule
& & \multicolumn{3}{c}{\textbf{EuroSAT} (10\,cls, 27k)} 
  & \multicolumn{3}{c}{\textbf{CIFAR-100} (100\,cls, 50k)} 
  & \multicolumn{3}{c}{\textbf{ImageNette} (10\,cls, 9.5k)}
  & \multicolumn{3}{c}{\textbf{BloodMNIST} (8\,cls, 12k)} \\
\cmidrule(lr){3-5}\cmidrule(lr){6-8}\cmidrule(lr){9-11}\cmidrule(lr){12-14}
& \textbf{Method}
  & $\alpha{=}0.1$ & $\alpha{=}0.05$ & Path.
  & $\alpha{=}0.1$ & $\alpha{=}0.05$ & Path.
  & $\alpha{=}0.1$ & $\alpha{=}0.05$ & Path.
  & $\alpha{=}0.1$ & $\alpha{=}0.05$ & Path. \\
\midrule
& FedSD2C
  & 75.83 & 68.47 & 54.62
  & 27.19 & 23.62 & 20.99
  & 57.24 & 48.36 & 35.82
  & 69.52 & 61.83 & 65.47 \\
& DENSE
  & 62.38 & 53.71 & 40.25
  & 10.61 & 5.55 & 3.82
  & 42.17 & 31.54 & 21.63
  & 58.46 & 48.72 & 54.83 \\
& CoBoosting
  & 58.62 & 49.83 & 36.47
  & 16.92 & 15.70 & 14.95
  & 39.36 & 28.53 & 19.27
  & 54.82 & 45.63 & 51.47 \\
& FedD3
  & 42.35 & 38.72 & 31.46
  & 1.82 & 1.74 & 1.61
  & 24.18 & 20.65 & 16.83
  & 48.36 & 43.85 & 50.27 \\
& FedCVAE
  & 40.83 & 37.45 & 30.62
  & 1.64 & 1.65 & 1.55
  & 25.47 & 21.38 & 17.52
  & 50.15 & 45.28 & 51.83 \\
\midrule
& \textbf{Ours (DP)}
  & \textbf{80.96} & \textbf{80.74} & \textbf{80.46}
  & \textbf{29.71} & \textbf{28.69} & \textbf{28.62}
  & \textbf{62.23} & \textbf{61.98} & \textbf{62.46}
  & \textbf{75.03} & \textbf{74.84} & \textbf{75.19} \\
\bottomrule
\end{tabular}
}
\end{table*}

% ─────────────────────────────────────────────────────────────────

\subsection{Communication and Compute}
\label{sec:efficiency}

Table~\ref{tab:compute_cost} profiles the client-side cost of FedKT-CSD for four 
autoencoder configurations. All timings are wall-clock measurements on a single 
NVIDIA RTX 5090 processing 1000 images. The pipeline has three stages: encoding, 
DP statistics computation and sampling, and decoding.

With our default configuration (DC-AE f32c32 at $64{\times}64$), the entire pipeline 
completes in 1.33 seconds, dominated by the encoder and decoder forward passes (0.47s 
and 0.85s respectively). The statistics and sampling step is negligible at 3ms, since 
it involves only vector sums and matrix operations in $\mathbb{R}^{128}$. The 
lightweight TAESD codec at $32{\times}32$ finishes in under 0.1 seconds total, making 
it viable even on CPU-only clients, though at the cost of lower downstream accuracy 
(Table~\ref{tab:ae_ablation}).

Upload cost is the total payload a client transmits to the server. A client sends one 
mean vector ($d$ floats) and one symmetric covariance matrix ($d(d{+}1)/2$ floats) per 
class. With $d{=}128$ and 10 classes this amounts to 327.5\,KB per client — 
approximately two orders of magnitude smaller than a single FedAvg model upload (a 
ResNet-18 is $\sim$45\,MB). The high-resolution DC-AE at $128{\times}128$ incurs a 
substantially larger upload of 5.0\,MB due to the $d{=}512$ latent dimension, which 
also degrades DP utility (Table~\ref{tab:ae_ablation}), making it unattractive on both 
axes. Peak GPU memory during encoding is 3.5\,GB for DC-AE f32c32, well within the 
capacity of consumer hardware. For deployment on edge devices without a GPU, the TAESD 
configuration requires only 333\,MB of GPU memory and 21\,MB of client RAM.

\begin{table}[t]
\centering
\small
\caption{\textbf{Client-side cost} (1000 images, 10 classes, RTX\,5090).
Zero training: encode $\to$ DP stats $\to$ sample $\to$ decode.}
\label{tab:compute_cost}
\begin{tabular}{l cccc}
\toprule
& \textbf{TAESD} & \textbf{DC-AE f32c32} & \textbf{DC-AE f64c128} & \textbf{DC-AE f32c32} \\
& 32$\times$32 & 64$\times$64 & 64$\times$64 & 128$\times$128 \\
\midrule
Enc./Dec.\ params       & 1.2\,M / 1.2\,M & 323\,M / 323\,M & 677\,M / 677\,M & 323\,M / 323\,M \\
Latent dim $d$           & 64       & 128      & 128       & 512 \\
Model weights            & 9\,MB    & 1.2\,GB  & 2.6\,GB   & 1.2\,GB \\
\midrule
Encode (s)               & 0.020    & 0.474    & 0.473     & 1.616 \\
DP stats + sample (s)    & 0.009    & 0.003    & 0.002     & 0.006 \\
Decode (s)               & 0.059    & 0.850    & 0.939     & 3.433 \\
\textbf{Total (s)}       & {0.09}  & 1.33  & 1.41  & 5.06 \\
\midrule
Peak GPU (encode)        & 333\,MB  & 3.5\,GB  & 4.8\,GB   & 10.1\,GB \\
Client RAM         & 21\,MB   & 1.3\,GB  & 2.6\,GB   & 1.4\,GB \\
\midrule
\textbf{Upload} (1$\times$) & 83.8\,KB & 327.5\,KB & 327.5\,KB  & 5.0\,MB \\
\bottomrule
\end{tabular}
\end{table}

\subsection{Visualization}
\begin{figure*}[t]
\centering
\includegraphics[width=0.95\textwidth]{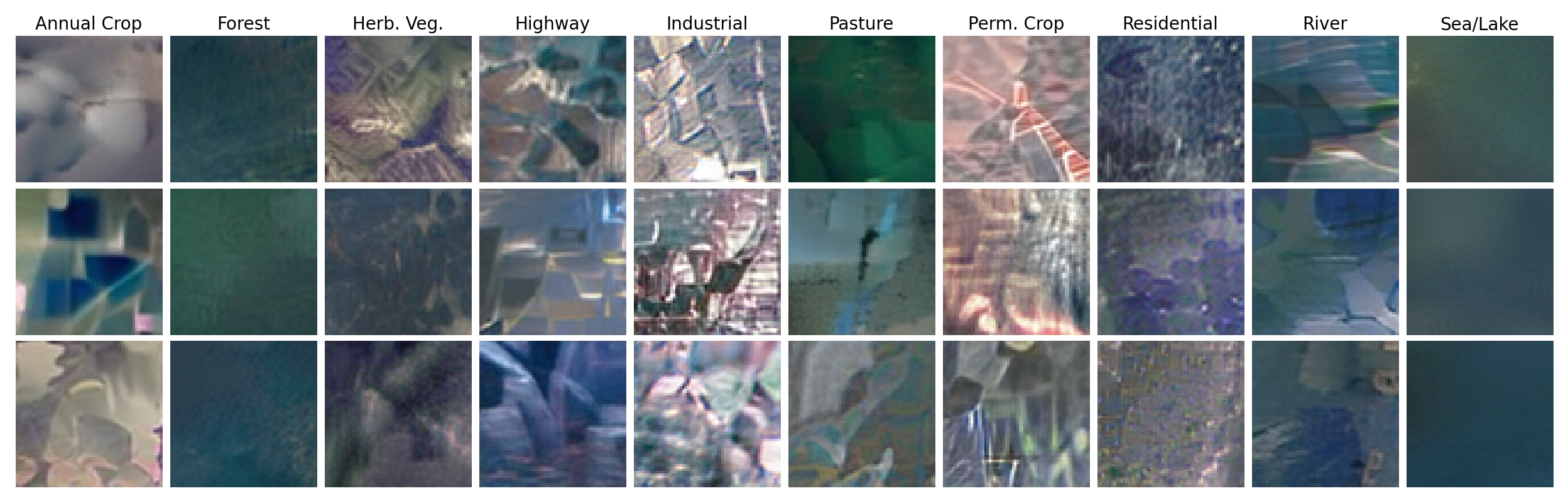}
\caption{Synthetic EuroSAT images generated by FedKT-CSD ($\varepsilon{=}10$, $\delta{=}10^{-5}$). Three random samples per class, decoded from DP-protected class-conditional Gaussians via the frozen DC-AE decoder.}
\label{fig:eurosat_samples}
\end{figure*}

Figure~\ref{fig:eurosat_samples} shows samples of synthetic EuroSAT images generated from the DP-protected class-conditional Gaussians. The samples are not entirely photorealistic, because each class is modeled as a single Gaussian in a 128-dimensional latent space, yet the decoder produces images that capture the dominant visual modes of each class rather than fine-grained spatial detail. Moreover, class-discriminative structure is clearly visible. Forests and pastures produce uniform green and blue tones, residential areas contain grid-like patterns with visible rooftops, highways exhibit linear features, and industrial regions display angular structures. Annual and permanent crops are distinguishable by texture regularity, while sea/lake samples show consistent dark blue hues. These visual differences align with what a downstream classifier needs to separate the classes, which explains why the synthetic data achieves $\sim$80\% accuracy on EuroSAT despite the coarse appearance of individual images. The key insight is that classification does not require photorealism; it requires that the synthetic distribution preserves the inter-class separation present in the real data, which the Gaussian model in latent space achieves effectively. We show more images of other dataset in the appendix.

% ─────────────────────────────────────────────────────────────────
\subsection{Comparison with other DP Synthetic Data Methods}
\label{sec:dp_synth}

Since we generate a DP synthetic dataset, we also compare FedKT-CSD against established DP synthetic image generation methods (see ~\cite{dpimagebench} for a comprehensive overview of DP methods for images) at $\varepsilon{=}10$, in both centralized and federated settings in Table~\ref{tab:dp_synth}. For the centralized setup we run the methods using the implementations in \cite{dpimagebench}. For the federated setup we simply run each method on the local client split and share the generated data to train a model on the server. 

The kernel-based methods (DP-Kernel~\cite{dp-kernel}, DP-MERF~\cite{dp-merf}, DP-NTK~\cite{dp-ntk}) operate by matching noisy statistics of random feature embeddings between real and synthetic data. The high-dimensional matching objectives require substantial noise at practical $\varepsilon$ values. FedKT-CSD outperforms all methods in the federated settings. This is because our pipeline sidesteps iterative optimization entirely. The pretrained DC-AE encoder provides a rich, low-dimensional embedding ($d{=}128$) in which simple Gaussian modeling already captures class structure well.

Furthermore, DP-LoRA LDM~\cite{dp-lora} takes a different approach, fine-tuning a latent diffusion model pretrained on the full ImageNet dataset using DP-SGD. In the centralized setting this is the strongest competitor, since the diffusion model can generate high-fidelity images and the full dataset provides enough gradient signal to absorb DP-SGD noise. However, the federated variant collapses on several datasets (e.g., 20\% on BloodMNIST). This is expected, as DP-SGD noise scales with the number of gradient steps and the sensitivity of each step, so when local datasets shrink in a federated split, the signal-to-noise ratio becomes prohibitive. FedKT-CSD does not suffer from this because it never performs gradient-based optimization on private data. The noise is injected once into low-dimensional additive statistics, and the total noise magnitude is independent of how data is partitioned across clients.

\begin{table}[t]
\centering
\small
\caption{\textbf{Comparison with DP synthetic data generation methods.}
All methods: $(\varepsilon{=}10,\delta{=}10^{-5})$-DP.
ResNet-18 trained on synthetic data, tested on real data.}
\label{tab:dp_synth}
\begin{tabular}{l l cccc}
\toprule
\textbf{Method} & \textbf{Setting}
  & \textbf{EuroSAT} & \textbf{CIFAR-100} & \textbf{ImageNette} & \textbf{BloodMNIST} \\
\midrule
\multicolumn{6}{l}{\emph{Centralized DP synthesis}} \\
DP-MERF       & Central & 33.70 & 3.60  & 15.81 & 49.38 \\
DP-NTK        & Central & 38.60 & 3.60  & 12.42 & 50.12 \\
DP-Kernel     & Central & 47.90 & 5.70  & 18.25 & 61.39 \\
DP-LoRA LDM   & Central & \textbf{84.60} & \textbf{35.20} & 58.37 & \textbf{80.06} \\
\midrule
\multicolumn{6}{l}{\emph{Federated DP synthesis (20 clients)}} \\
DP-Kernel     & Fed.    & 28.35 & 7.16  & 16.42 & 52.62 \\
DP-MERF       & Fed.    & 22.41 & 4.85  & 11.93 & 39.62 \\
DP-NTK        & Fed.    & 19.83 & 4.52  & 11.27 & 48.75 \\
DP-LoRA LDM   & Fed.    & 31.56 & 9.24  & 22.18 & 19.76 \\
\midrule
\multicolumn{6}{l}{\emph{Ours}} \\
TAESD              & Fed.  & 76.27 & 27.08 & 48.63 & 74.59 \\
DC-AE f32c32       & Fed.  & 80.68 & 29.34 & \textbf{61.10} & 74.78 \\
\bottomrule
\end{tabular}
\end{table}

\subsection{Downstream Tasks}
\label{sec:personalization}

A key advantage of our approach is that the synthetic dataset $X^*$ consists of actual images, not features tied to a specific model. Methods such as FedPFT~\cite{fedpft} generate synthetic representations in a fixed feature space, meaning any downstream model must use the same feature extractor. By contrast, $X^*$ is model-agnostic and can train any architecture, serve as a pretraining corpus, or be combined with local data for fine-tuning.

We demonstrate this flexibility by using $X^*$ as a pretraining resource for personalized federated learning, following the pFL-Bench protocol~\cite{pflib}. We use 20 clients with full participation and a ResNet-18 backbone. Baselines train via multi-round optimization (50 rounds). For FedKT-CSD, the server pretrains ResNet-18 on $X^*$, distributes it, and each client fine-tunes only the classification layer locally. Crucially, this ResNet-18 is entirely separate from the autoencoder that generated $X^*$, i.e., the synthetic images transfer across model families.

Table~\ref{table:pers_osfl} shows results on CIFAR-10 and CIFAR-100 under pathological and practical heterogeneity. The value of global synthetic pretraining scales with the severity of label-space fragmentation. On CIFAR-100 pathological, each client observes a narrow class slice and multi-round methods cannot recover structure that was never represented locally. Our globally aggregated statistics reconstruct the full class-conditional distribution, yielding a nearly 9-point lead. This advantage narrows on CIFAR-10 and under practical heterogeneity for the same reason. As local class overlap increases, the information deficit shrinks and local optimization alone suffices.

\begin{table}[t]
\centering
\caption{Downstream personalized FL comparison with multi-round pFL methods. All methods use ResNet-18 with 20 clients and full participation. Baselines run 50 communication rounds; FedKT-CSD uses one round (synthetic pretraining) followed by local classifier fine-tuning.}
\resizebox{0.75\textwidth}{!}{%
\begin{tabular}{l*{4}{c}}
\toprule
& \multicolumn{2}{c}{Pathological Heterogeneity} & \multicolumn{2}{c}{Practical Heterogeneity} \\
\cmidrule(lr){2-3} \cmidrule(lr){4-5}
& CIFAR-10 & CIFAR-100 & CIFAR-10 & CIFAR-100 \\
\midrule
FedAvg \cite{FedAvg}         & 86.05 & 44.88 & 79.68 & 32.77 \\
Per-FedAvg \cite{Per-FedAvg} & 90.12 & 55.98 & 87.93 & 45.38 \\
pFedMe \cite{pFedMe}         & 91.21 & 59.88 & 89.78 & 48.12 \\
FedAMP \cite{FedAMP}         & 89.85 & 63.32 & 89.14 & 49.68 \\
FedPHP \cite{FedPHP}         & 89.12 & 61.45 & 87.69 & 51.21 \\
FedFomo \cite{FedFOMO}       & 90.78 & 61.49 & 87.54 & 46.21 \\
APPLE \cite{APPLE}           & 90.02 & 64.84 & 88.57 & 52.14 \\
PartialFed \cite{PartialFed} & 88.95 & 60.41 & 87.29 & 49.85 \\
\midrule
FedKT-CSD & \textbf{92.82} & \textbf{73.44} & \textbf{88.78} & \textbf{59.89} \\
\bottomrule
\end{tabular}
}
\label{table:pers_osfl}
\end{table}

% ─────────────────────────────────────────────────────────────────

% ─────────────────────────────────────────────────────────────────

\section{Privacy Evaluation}
\label{sec:privacy_eval}

\paragraph{Membership inference attack.}
Although the formal DP guarantee  already bounds any adversary's advantage, we include a membership inference evaluation \cite{MIA} for comparability with prior OSFL methods that rely on MIA as their primary privacy metric. Figure~\ref{fig:ablation_mia} shows the ROC curve. The AUC is 0.501, indistinguishable from random guessing, which is the optimal result.

This result is naturally follows from the structure of our framework. Unlike standard MIA settings where the attacker probes a model that was directly trained on private data and may have memorized individual examples, our synthetic images are samples from a Gaussian distribution fit to aggregated statistics with DP. To detect whether a specific image participated, an attacker would need its individual contribution to remain detectable after aggregation across hundreds of samples, addition of calibrated DP noise, random sampling from the resulting distribution, and decoding through a public network. Each of these steps dilutes per-record influence, making membership inference structurally difficult.

\paragraph{Privacy-utility tradeoff.}
Figure~\ref{fig:ablation_privacy} shows downstream accuracy on ImageNette as a function of the privacy budget $\varepsilon$ (with $\delta{=}10^{-5}$). At $\varepsilon{=}1$, heavy noise reduces accuracy to roughly 25\%. Performance climbs steeply through $\varepsilon{=}3$--$5$ and reaches $\sim$62\% at our default $\varepsilon{=}10$, with diminishing returns beyond that (the non-DP centralized upper bound is $\sim$65\%). The steep gains at moderate $\varepsilon$ reflect the low dimensionality of the latent statistics. Because noise is injected into $d{=}128$ dimensional vectors rather than high-dimensional model parameters, even a modest privacy budget yields a favorable signal-to-noise ratio.

\section{Ablation of Method and Experimental Parameters}
% Figure 2: Ablation Studies
\begin{figure*}[t]
\centering

% Shared legend across top
\includegraphics[width=0.75\textwidth]{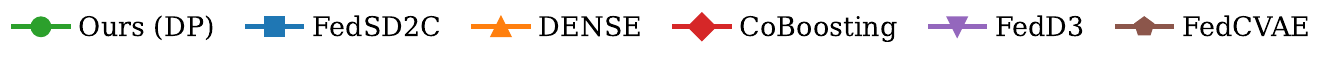}

\vspace{-0.3em}

\begin{subfigure}[b]{0.32\textwidth}
    \centering
    \includegraphics[width=\textwidth]{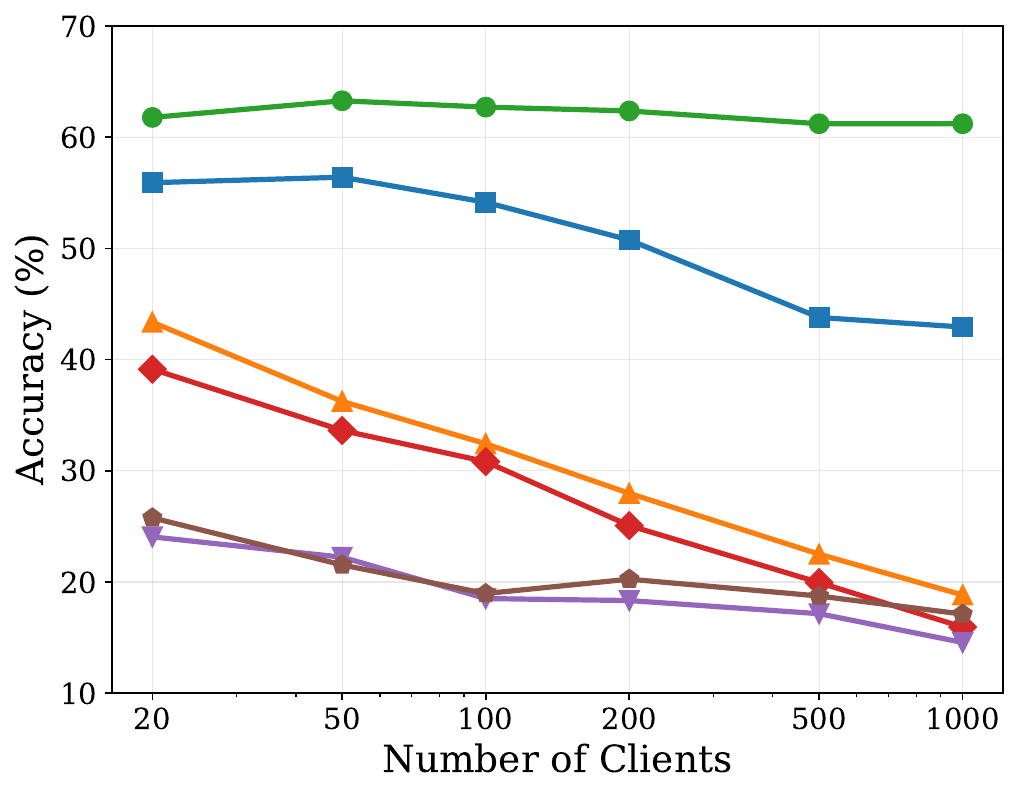}
    \caption{Number of clients}
    \label{fig:ablation_clients}
\end{subfigure}
\hfill
\begin{subfigure}[b]{0.32\textwidth}
    \centering
    \includegraphics[width=\textwidth]{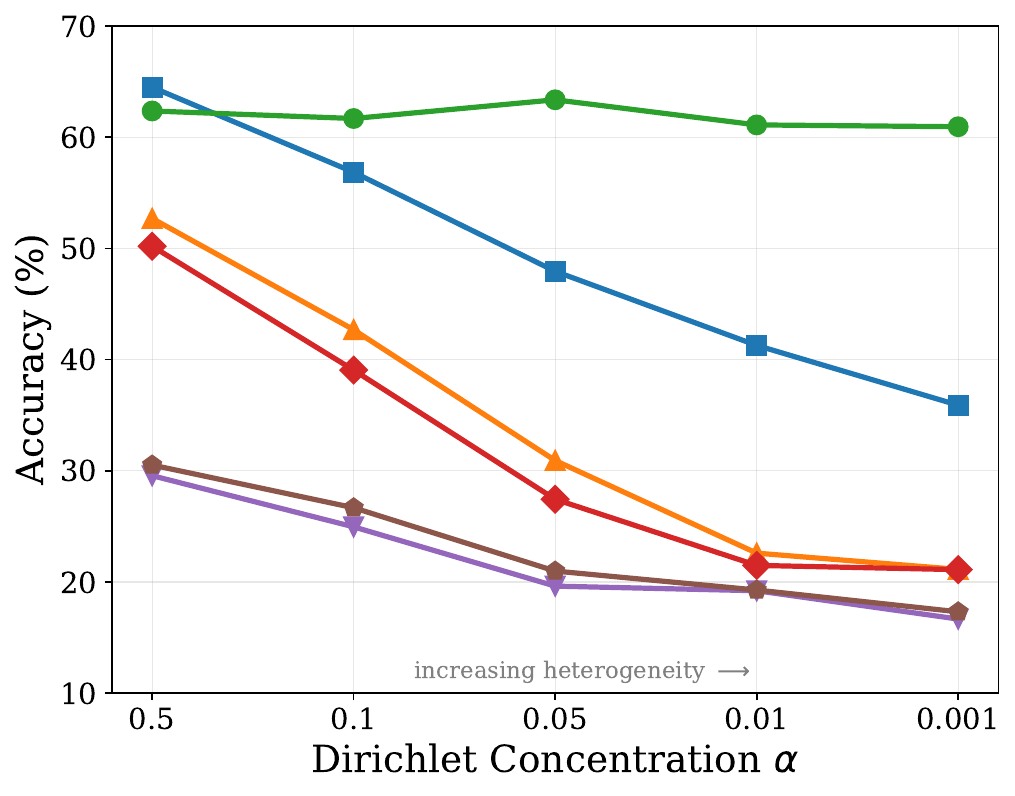}
    \caption{Heterogeneity (Dirichlet $\alpha$)}
    \label{fig:ablation_heterogeneity}
\end{subfigure}
\hfill
\begin{subfigure}[b]{0.32\textwidth}
    \centering
    \includegraphics[width=\textwidth]{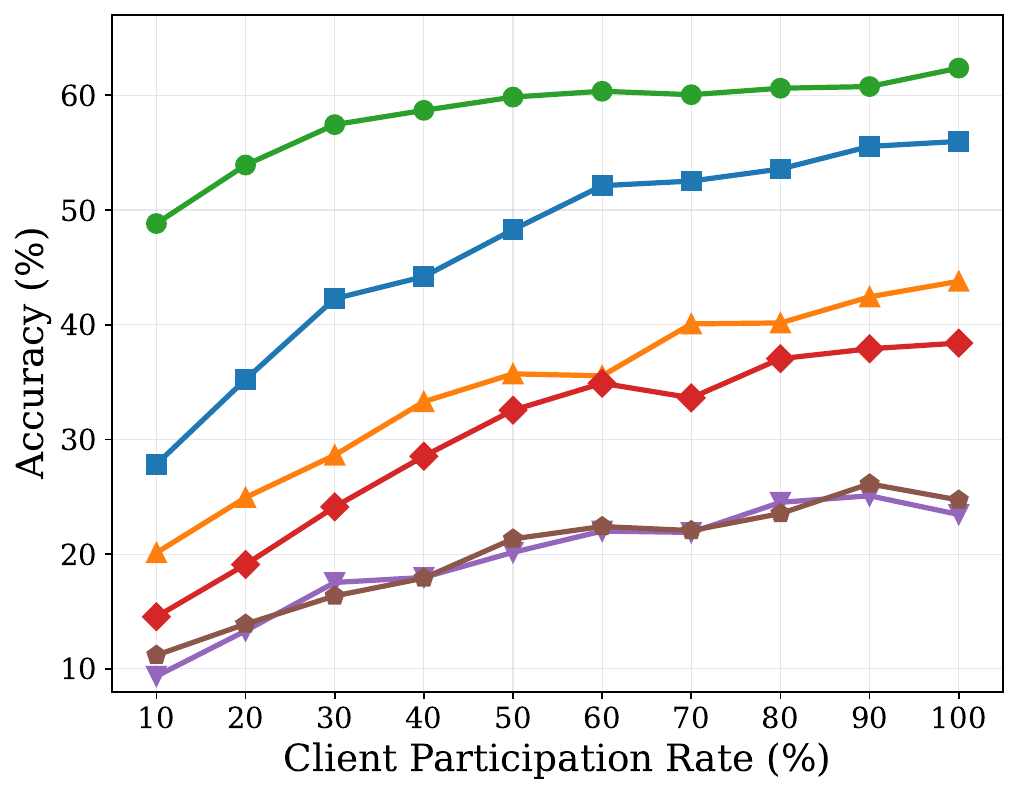}
    \caption{Client participation rate}
    \label{fig:ablation_participation}
\end{subfigure}

\vspace{0.3em}

\begin{subfigure}[b]{0.32\textwidth}
    \centering
    \includegraphics[width=\textwidth]{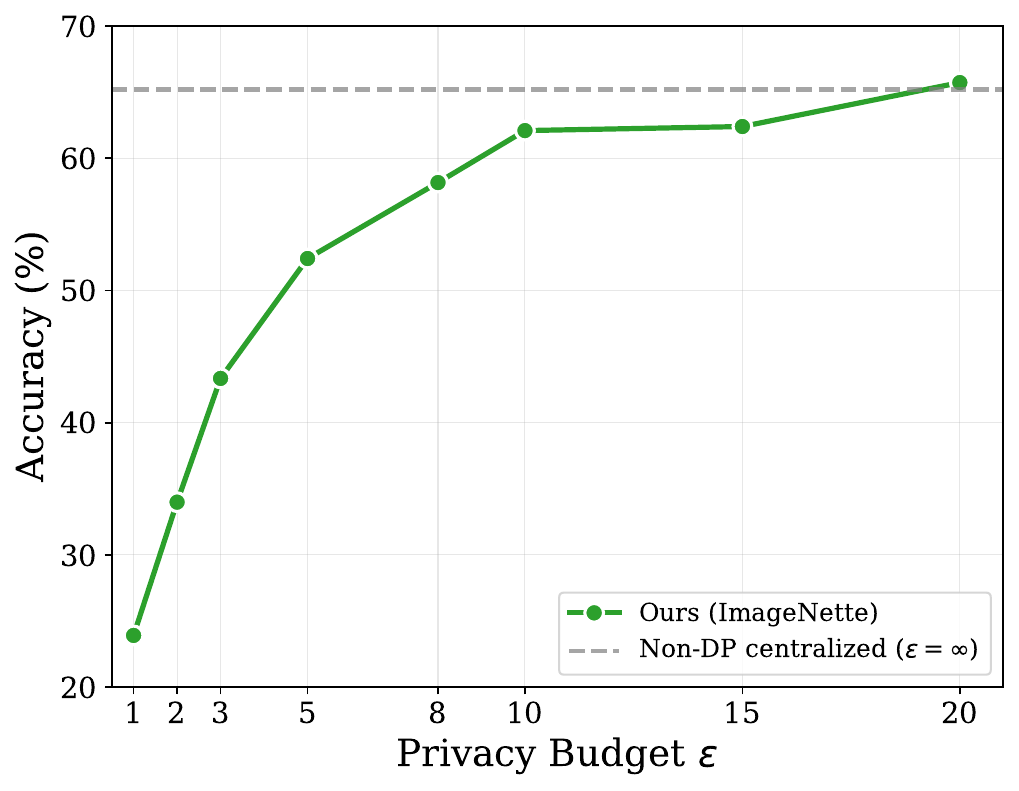}
    \caption{Privacy budget $\varepsilon$}
    \label{fig:ablation_privacy}
\end{subfigure}
\hfill
\begin{subfigure}[b]{0.32\textwidth}
    \centering
    \includegraphics[width=\textwidth]{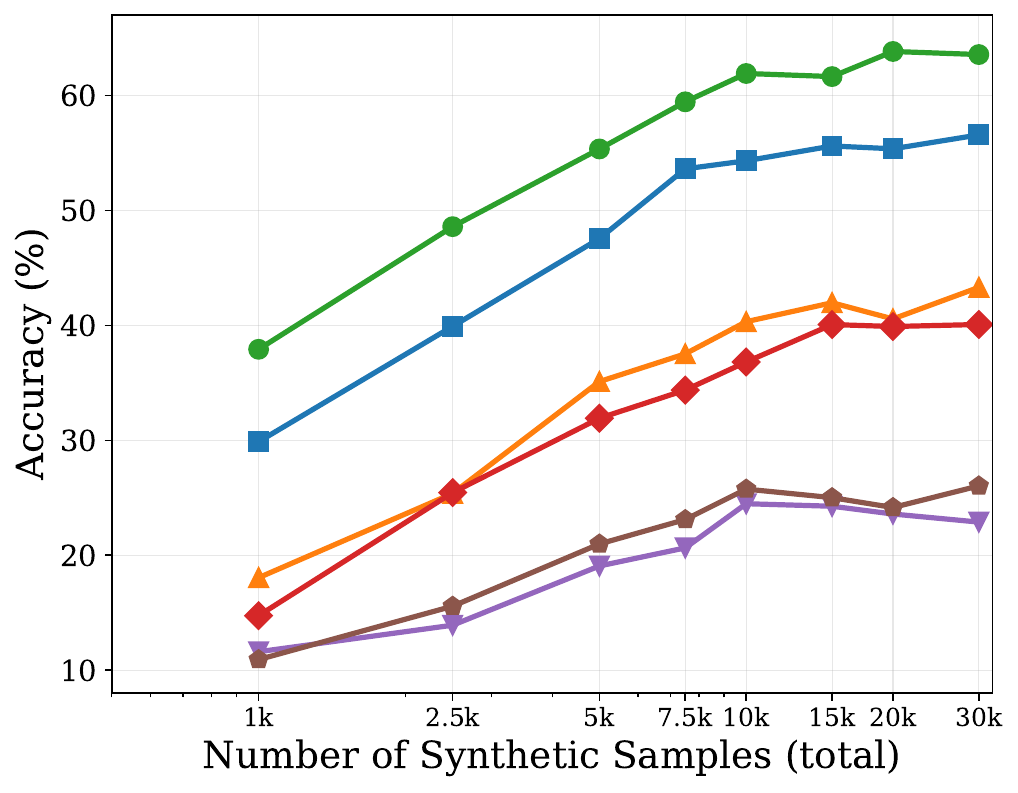}
    \caption{Number of synthetic samples}
    \label{fig:ablation_samples}
\end{subfigure}
\hfill
\begin{subfigure}[b]{0.32\textwidth}
    \centering
    \includegraphics[width=\textwidth]{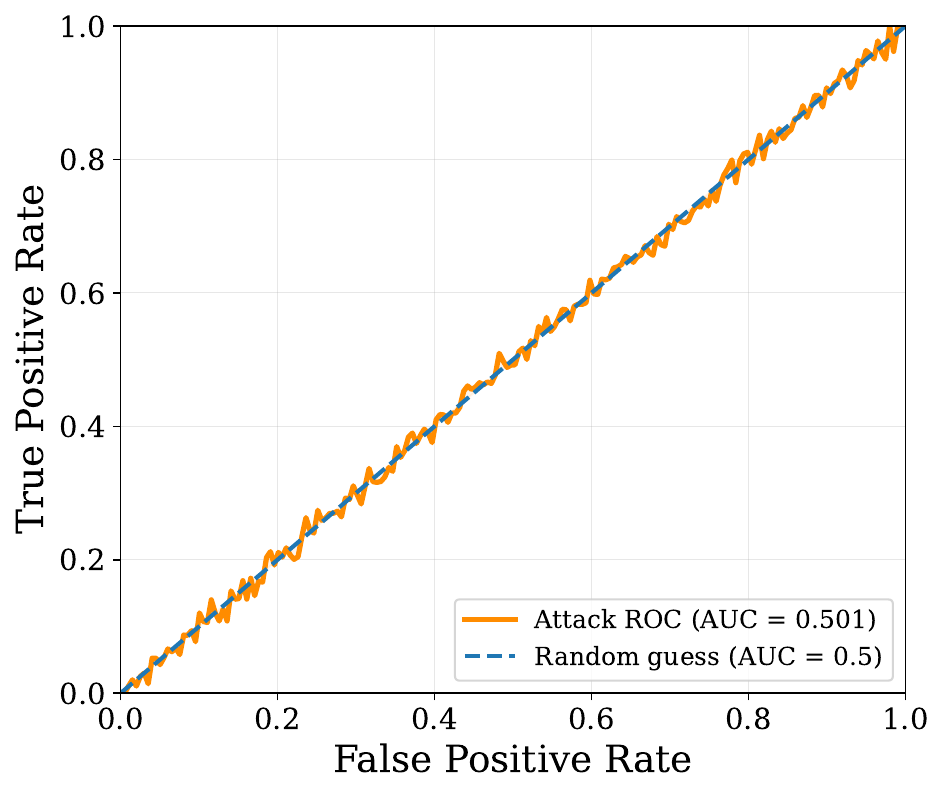}
    \caption{Membership inference attack}
    \label{fig:ablation_mia}
\end{subfigure}

\caption{Ablation studies on ImageNette. (a) Different number of clients. (b) Different heterogeneity settings(c)~Participation rate (d)~Privacy-utility tradeoff for our method at varying $\varepsilon$. (e) Different number of synthetic samples. (f)~MIA ROC curve.}
\label{fig:ablations}
\end{figure*}

% \begin{table}[!ht]
% \centering
% \caption{Ablations varying number of local epochs $E$, latent space dimension $d$, and number of total samples generated $N_{Tot}$ for server-side training for CIFAR100 under practical heterogeneity.}
% \label{tab:various_ablations}
% \begin{tabular}{cc|cc|cc|cc}
% \toprule
% Epochs & Acc. & Latent Dimension & Acc. & Number of Samples & Acc. &  Number of Clients & Acc. \\
% \midrule
% 10   & 31.79 & 50   & 56.77 & 500   & 52.43 & 20   & 59.89 \\
% 30   & 42.26 & 100  & 57.65 & 1000  & 54.87 & 40   & 58.90 \\
% 50   & 54.60 & 150  & 58.88 & 1500  & 56.06 & 80   & 56.78 \\
% 70   & 58.51 & 200  & 59.89 & 2000  & 59.89 & 100   & 55.21 \\
% 90   & 59.38 & 250  & 59.72 & 2500  & 58.33 & 200   & 54.67 \\
% 100  & 59.89 & 300  & 57.13 & 3000  & 54.82 & 500   & 52.37 \\
% \bottomrule
% \end{tabular}
% \vspace{-3mm}
% \end{table}

\label{sec:ablations}

We ablate key design choices on ImageNette with 20 clients under practical heterogeneity ($\alpha{=}0.1$) unless stated otherwise. All baselines are run without DP. Results are shown in Figure~\ref{fig:ablations}.

\paragraph{Number of clients.}
Figure~\ref{fig:ablation_clients} varies the federation size from 20 to 1000. All baselines degrade substantially as data is spread more thinly: FedSD2C drops from 57\% to 42\%, and weaker methods fall below 20\%. Our method is invariant to federation size because the global class-conditional sums are identical regardless of how data is partitioned across clients. The small fluctuations in our curve ($\sim$1\%) are purely due to downstream training variance.

\paragraph{Heterogeneity.}
Figure~\ref{fig:ablation_heterogeneity} varies the Dirichlet concentration from $\alpha{=}0.5$ (mild heterogeneity) to $\alpha{=}0.001$ (extreme). Every baseline degrades monotonically: FedSD2C drops from $\sim$64\% to $\sim$36\%, and the other methods can drop below 20\%. Our method stays at $\sim$62\% across all settings because per-class statistics are aggregated identically regardless of how labels are distributed across clients. This heterogeneity invariance is a direct consequence of the additive structure of the pipeline.

\paragraph{Client participation rate.}
Figure~\ref{fig:ablation_participation} varies the fraction of clients that participate in a single round from 10\% to 100\%. Our method degrades gracefully. Even at 10\% participation (2 out of 20 clients), accuracy remains above 48\%, compared to 28\% for FedSD2C and below 20\% for the other baseline methods. The degradation in our case is small because the statistics from participating clients still combine additively. The main effect of lower participation is fewer samples contributing to the global sums, which increases the relative magnitude of DP noise.

\paragraph{Number of synthetic samples.}
Figure~\ref{fig:ablation_samples} varies the total number of synthetic images generated from the DP-protected Gaussians. Accuracy climbs from $\sim$39\% at 1k samples to $\sim$62\% at 10k, then plateaus. The saturation point roughly matches the size of the real dataset (9.5k images), which is expected; once the synthetic dataset is large enough to cover the learned Gaussian distributions, additional samples provide diminishing returns. Baselines follow a similar saturation pattern but at their respectively lower ceilings.

\paragraph{Autoencoder Architecture Ablation}
\label{sec:ablation}

Table~\ref{tab:ae_ablation} ablates the autoencoder architecture. The key design question is choosing the latent dimension $d$. Higher $d$ gives the encoder more capacity to represent images, however the DP noise injected into each coordinate accumulates across dimensions. Specifically, when independent Gaussian noise of scale $\sigma_\mu$ is added to a $d$-dimensional mean vector, the total $L_2$ magnitude of the noise increases as $\sigma_\mu\sqrt{d} \propto C\sqrt{d}$, where $C$ is the per-class clipping sensitivity. We report $C\sqrt{d}$ as a signal-to-noise proxy. Lower values mean less total noise relative to the statistics.

The table reveals a clear tradeoff. At $d{=}512$ (the raw DC-AE representation at $128{\times}128$ resolution), $C\sqrt{d}{=}1934$ and accuracy collapses to 38--58\% across datasets. The latent space is expressive but noise strongly dominates as it overlays the signal. At $d{=}128$ with $C\sqrt{d} \approx 250$--$400$, accuracy peaks, showing that the latent space is expressive enough and the noise is moderate. Reducing to $d{=}64$ via PCA further lowers $C\sqrt{d}$ to 241, but discarding half the latent dimensions sacrifices representational capacity, resulting in 3--7\% lower accuracy. At $d{=}32$ the noise is minimal ($C\sqrt{d}{=}119$) but the bottleneck is too narrow to capture class structure.

A second trend is that resolution and reconstruction quality matter independently of $d$. DC-AE at $64{\times}64$ consistently outperforms TAESD at $32{\times}32$ despite both using $d{=}64$--$128$, because higher-resolution decoding recovers more visual detail for the downstream classifier. The best overall configuration is DC-AE f32c32-in at $64{\times}64$ with $d{=}128$, which balances representational capacity, reconstruction quality, and noise tolerance.

\begin{table}[t]
\centering
\caption{\textbf{Autoencoder architecture ablation} ($\varepsilon{=}10, \delta{=}10^{-5}$).
$d$ = latent dim; $C$ = per-class sensitivity; $C\sqrt{d}$ = noise-signal proxy.}
\label{tab:ae_ablation}
\resizebox{\textwidth}{!}{%
\begin{tabular}{l c c c c ccccc}
\toprule
\textbf{Codec} & \textbf{Res.} & $d$ & $C$ & $C\sqrt{d}$
  & \textbf{EuroSAT} & \textbf{CIFAR-10} & \textbf{CIFAR-100} & \textbf{ImageNette} & \textbf{BloodMNIST} \\
\midrule
TAESD                      & 32  & 64  & 7--8   & 57--68
  & 76.83 & 70.86 & 24.70 & 47.82 & 72.99 \\
\midrule
DC-AE f32c32-in            & 64  & 128 & 34.15  & 387
  & \textbf{80.59} & \textbf{73.80} & \textbf{28.95} & 60.83 & \textbf{74.73} \\
DC-AE f64c128-mix          & 64  & 128 & 24.45  & 277
  & 79.05 & 73.33 & 28.48 & 61.23 & 74.66 \\
DC-AE f32c32-mix           & 64  & 128 & 21.90  & 248
  & 79.13 & 73.53 & 28.18 & 59.76 & 74.70 \\
DC-AE f32c32-in+PCA64      & 64  & 64  & 30.09  & 241
  & 79.33 & 70.77 & 26.13 & 54.78 & 72.13 \\
DC-AE f32c32-in+PCA128     & 128 & 128 & 71.07  & 804
  & 78.39 & 70.14 & 27.39 & 62.38 & 73.07 \\
DC-AE f32c32-in (raw)      & 128 & 512 & 85.50  & 1934
  & 58.60 & 49.14 & 14.37 & 38.77 & 51.67 \\
DC-AE f32c32-in            & 32  & 32  & 20.98  & 119
  & 75.03 & 67.11 & 25.91 & 53.02 & 74.34 \\
\bottomrule
\end{tabular}
}
\end{table}
\section{Limitations}

Our method assumes labeled classification tasks where each client knows the class identity of its samples. Extending to unsupervised or self-supervised settings, where no class labels are available to condition the statistics on, remains an open problem for our approach and for data-sharing methods in FL more broadly. Furthermore, the Gaussian model assumes reasonable class separability in the latent space. Under noisy labels or highly overlapping class distributions, the per-class statistics would degrade, though this sensitivity is shared by most class-conditional generative approaches. Moreover, single Gaussian per class could limit the quality of generated data. However, while a GMM would offer more expressive power, it would require per-component assignments that are not additively aggregable across clients without revealing additional information. Finally, classes with very few samples yield noisy statistics. With $n_c$ small, the effective noise in the mean scales as $\sigma_\mu / n_c$, which can dominate through overlaying the signal. This is, however, inherent to DP mechanisms that calibrate noise to worst-case sensitivity.

\section{Conclusion}

We have presented FedKT-CSD, a one-shot federated learning framework that generates DP synthetic data by aggregating class-conditional latent statistics from a frozen pretrained autoencoder. The method requires a single communication round, no on-device training, and uploads in the KB regime per client, while providing formal $(\varepsilon, \delta)$-DP guarantees with clean accounting via the Gaussian mechanism. Empirically, FedKT-CSD outperforms existing one-shot FL baselines that operate without any privacy protection, is competitive with or superior to centralized DP synthetic data methods, and scales to 1000+ clients with no performance degradation. The synthetic dataset is versatile: we demonstrate its use for both centralized classifier training and downstream personalized FL, where it achieves state-of-the-art results from a single round of communication. The core insight, that pretrained autoencoders provide a shared low-dimensional space in which simple additive statistics suffice for high-quality data generation, suggests a broader design principle for privacy-preserving federated systems.

% \subsubsection*{Acknowledgments}
% Use unnumbered third level headings for the acknowledgments. All
% acknowledgments, including those to funding agencies, go at the end of the paper.
% Only add this information once your submission is accepted and deanonymized. 

\bibliography{tmlr}

@inproceedings{FedAvg,
  title={Communication-efficient learning of deep networks from decentralized data},
  author={McMahan, H. Brendan and Moore, Eider and Ramage, Daniel and Hampson, Seth and y Arcas, Blaise Aguera},
  booktitle={Proceedings of the 20th International Conference on Artificial Intelligence and Statistics (AISTATS)},
  pages={1273--1282},
  year={2017},
  organization={PMLR}
}

@article{data_sharing_FL,
  title={Federated learning with non-iid data},
  author={Zhao, Yue and Li, Meng and Lai, Liangzhen and Suda, Naveen and Civin, Damon and Chandra, Vikas},
  journal={arXiv preprint arXiv:1806.00582},
  year={2018}
}

@article{Dense,
  title={Dense: Data-free one-shot federated learning},
  author={Zhang, Jie and Chen, Chen and Li, Bo and Lyu, Lingjuan and Wu, Shuang and Ding, Shouhong and Shen, Chunhua and Wu, Chao},
  journal={Advances in Neural Information Processing Systems},
  volume={35},
  pages={21414--21428},
  year={2022}
}

@inproceedings{Per-FedAvg,
 author = {Fallah, Alireza and Mokhtari, Aryan and Ozdaglar, Asuman},
 booktitle = {Advances in Neural Information Processing Systems},
 editor = {H. Larochelle and M. Ranzato and R. Hadsell and M.F. Balcan and H. Lin},
 pages = {3557--3568},
 publisher = {Curran Associates, Inc.},
 title = {Personalized Federated Learning with Theoretical Guarantees: A Model-Agnostic Meta-Learning Approach},
 url = {https://proceedings.neurips.cc/paper_files/paper/2020/file/24389bfe4fe2eba8bf9aa9203a44cdad-Paper.pdf},
 volume = {33},
 year = {2020}
}

@inproceedings{pFedMe,
  title={Personalized federated learning with Moreau envelopes},
  author={Dinh, Canh T. and Tran, Nguyen H. and Nguyen, Tuan D.},
  booktitle={Advances in Neural Information Processing Systems (NeurIPS)},
  volume={33},
  pages={21394--21405},
  year={2020}
}

@inproceedings{
FedFOMO,
title={Personalized Federated Learning with First Order Model Optimization},
author={Michael Zhang and Karan Sapra and Sanja Fidler and Serena Yeung and Jose M. Alvarez},
booktitle={International Conference on Learning Representations},
year={2021},
url={https://openreview.net/forum?id=ehJqJQk9cw}
}

@InProceedings{FedPHP,
author="Li, Xin-Chun
and Zhan, De-Chuan
and Shao, Yunfeng
and Li, Bingshuai
and Song, Shaoming",
editor="Oliver, Nuria
and P{\'e}rez-Cruz, Fernando
and Kramer, Stefan
and Read, Jesse
and Lozano, Jose A.",
title="FedPHP: Federated Personalization with Inherited Private Models",
booktitle="Machine Learning and Knowledge Discovery in Databases. Research Track",
year="2021",
publisher="Springer International Publishing",
address="Cham",
pages="587--602",
isbn="978-3-030-86486-6"
}

@inproceedings{PartialFed,
 author = {Sun, Benyuan and Huo, Hongxing and YANG, YI and Bai, Bo},
 booktitle = {Advances in Neural Information Processing Systems},
 editor = {M. Ranzato and A. Beygelzimer and Y. Dauphin and P.S. Liang and J. Wortman Vaughan},
 pages = {23309--23320},
 publisher = {Curran Associates, Inc.},
 title = {PartialFed: Cross-Domain Personalized Federated Learning via Partial Initialization},
 url = {https://proceedings.neurips.cc/paper_files/paper/2021/file/c429429bf1f2af051f2021dc92a8ebea-Paper.pdf},
 volume = {34},
 year = {2021}
}

@article{FedAMP, title={Personalized Cross-Silo Federated Learning on Non-IID Data}, volume={35}, url={https://ojs.aaai.org/index.php/AAAI/article/view/16960}, DOI={10.1609/aaai.v35i9.16960}, number={9}, journal={Proceedings of the AAAI Conference on Artificial Intelligence}, author={Huang, Yutao and Chu, Lingyang and Zhou, Zirui and Wang, Lanjun and Liu, Jiangchuan and Pei, Jian and Zhang, Yong}, year={2021}, month={May}, pages={7865-7873} }

@inproceedings{APPLE,
  title={Adapt to adaptation: Learning personalization for cross-silo federated learning},
  author={Luo, Jun and Wu, Shandong},
  booktitle={IJCAI: proceedings of the conference},
  volume={2022},
  pages={2166},
  year={2022}
}

@inproceedings{CCVR,
author = {Luo, Mi and Chen, Fei and Hu, Dapeng and Zhang, Yifan and Liang, Jian and Feng, Jiashi},
title = {No fear of heterogeneity: classifer calibration for federated learning with non-IID data},
year = {2024},
isbn = {9781713845393},
publisher = {Curran Associates Inc.},
address = {Red Hook, NY, USA},
booktitle = {Proceedings of the 35th International Conference on Neural Information Processing Systems},
articleno = {457},
numpages = {13},
series = {NIPS '21}
}

@article{FedFed,
  title={Fedfed: Feature distillation against data heterogeneity in federated learning},
  author={Yang, Zhiqin and Zhang, Yonggang and Zheng, Yu and Tian, Xinmei and Peng, Hao and Liu, Tongliang and Han, Bo},
  journal={Advances in neural information processing systems},
  volume={36},
  pages={60397--60428},
  year={2023}
}

@InProceedings{FedFTG,
    author    = {Zhang, Lin and Shen, Li and Ding, Liang and Tao, Dacheng and Duan, Ling-Yu},
    title     = {Fine-Tuning Global Model via Data-Free Knowledge Distillation for Non-IID Federated Learning},
    booktitle = {Proceedings of the IEEE/CVF Conference on Computer Vision and Pattern Recognition (CVPR)},
    month     = {June},
    year      = {2022},
    pages     = {10174-10183}
}

@InProceedings{FedGen,
  title = 	 {Data-Free Knowledge Distillation for Heterogeneous Federated Learning},
  author =       {Zhu, Zhuangdi and Hong, Junyuan and Zhou, Jiayu},
  booktitle = 	 {Proceedings of the 38th International Conference on Machine Learning},
  pages = 	 {12878--12889},
  year = 	 {2021},
  editor = 	 {Meila, Marina and Zhang, Tong},
  volume = 	 {139},
  series = 	 {Proceedings of Machine Learning Research},
  month = 	 {18--24 Jul},
  publisher =    {PMLR},
  url = 	 {https://proceedings.mlr.press/v139/zhu21b.html}
}

@inproceedings{dworkDP,
author = {Dwork, Cynthia},
title = {Differential privacy},
year = {2006},
isbn = {3540359079},
publisher = {Springer-Verlag},
address = {Berlin, Heidelberg},
url = {https://doi.org/10.1007/11787006_1},
doi = {10.1007/11787006_1},
booktitle = {Proceedings of the 33rd International Conference on Automata, Languages and Programming - Volume Part II},
pages = {1–12},
numpages = {12},
location = {Venice, Italy},
series = {ICALP'06}
}

@article{dwork2014algorithmic,
author = {Dwork, Cynthia and Roth, Aaron},
title = {The Algorithmic Foundations of Differential Privacy},
year = {2014},
issue_date = {Aug 2014},
publisher = {Now Publishers Inc.},
address = {Hanover, MA, USA},
volume = {9},
number = {3–4},
issn = {1551-305X},
url = {https://doi.org/10.1561/0400000042},
doi = {10.1561/0400000042},
journal = {Found. Trends Theor. Comput. Sci.},
month = aug,
pages = {211–407},
numpages = {197}
}

@INPROCEEDINGS{FedD3,
  author={Song, Rui and Liu, Dai and Chen, Dave Zhenyu and Festag, Andreas and Trinitis, Carsten and Schulz, Martin and Knoll, Alois},
  booktitle={2023 International Joint Conference on Neural Networks (IJCNN)}, 
  title={Federated Learning via Decentralized Dataset Distillation in Resource-Constrained Edge Environments}, 
  year={2023},
  volume={},
  number={},
  pages={1-10},
  doi={10.1109/IJCNN54540.2023.10191879}}

@inproceedings{coboosting,
title={Enhancing One-Shot Federated Learning Through Data and Ensemble Co-Boosting},
author={Rong Dai and Yonggang Zhang and Ang Li and Tongliang Liu and Xun Yang and Bo Han},
booktitle={The Twelfth International Conference on Learning Representations},
year={2024},
url={https://openreview.net/forum?id=tm8s3696Ox}
}

@InProceedings{FedRep,
  title = 	 {Exploiting Shared Representations for Personalized Federated Learning},
  author =       {Collins, Liam and Hassani, Hamed and Mokhtari, Aryan and Shakkottai, Sanjay},
  booktitle = 	 {Proceedings of the 38th International Conference on Machine Learning},
  pages = 	 {2089--2099},
  year = 	 {2021},
  editor = 	 {Meila, Marina and Zhang, Tong},
  volume = 	 {139},
  series = 	 {Proceedings of Machine Learning Research},
  month = 	 {18--24 Jul},
  publisher =    {PMLR},
  url = 	 {https://proceedings.mlr.press/v139/collins21a.html}
}

@inproceedings{analytic_gaussian,
  title={Improving the gaussian mechanism for differential privacy: Analytical calibration and optimal denoising},
  author={Balle, Borja and Wang, Yu-Xiang},
  booktitle={International conference on machine learning},
  pages={394--403},
  year={2018},
  organization={PMLR}
}

@inproceedings{
one-shot-cvae,
title={Data-Free One-Shot Federated Learning Under Very High Statistical Heterogeneity},
author={Clare Elizabeth Heinbaugh and Emilio Luz-Ricca and Huajie Shao},
booktitle={The Eleventh International Conference on Learning Representations },
year={2023},
url={https://openreview.net/forum?id=_hb4vM3jspB}
}

@inproceedings{Geirhos2019TextureBias,
  author    = {Robert Geirhos and Patricia Rubisch and Claudio Michaelis and Matthias Bethge and Felix A. Wichmann and Wieland Brendel},
  title     = {ImageNet-trained CNNs are biased towards texture; increasing shape bias improves accuracy and robustness},
  booktitle = {International Conference on Learning Representations (ICLR)},
  year      = {2019},
  url       = {https://arxiv.org/abs/1811.12231}
}

@inproceedings{Nguyen2015EasilyFooled,
  author    = {Anh Nguyen and Jason Yosinski and Jeff Clune},
  title     = {Deep Neural Networks Are Easily Fooled: High Confidence Predictions for Unrecognizable Images},
  booktitle = {IEEE Conference on Computer Vision and Pattern Recognition (CVPR)},
  pages     = {427--436},
  year      = {2015},
  doi       = {10.1109/CVPR.2015.7298640}
}

@inproceedings{Wu2019LowResTransfer,
  author    = {Yuanwei Wu and Ziming Zhang and Guanghui Wang},
  title     = {Unsupervised Deep Feature Transfer for Low Resolution Image Classification},
  booktitle = {IEEE/CVF International Conference on Computer Vision Workshops (ICCVW)},
  pages     = {1848--1857},
  year      = {2019},
  url       = {https://arxiv.org/abs/1908.10012}
}

@inproceedings{Yosinski2014Transferable,
author = {Yosinski, Jason and Clune, Jeff and Bengio, Yoshua and Lipson, Hod},
title = {How transferable are features in deep neural networks?},
year = {2014},
publisher = {MIT Press},
address = {Cambridge, MA, USA},
booktitle = {Proceedings of the 28th International Conference on Neural Information Processing Systems - Volume 2},
pages = {3320–3328},
numpages = {9},
location = {Montreal, Canada},
series = {NIPS'14}
}

@inproceedings{Zeiler2014Visualizing,
  author    = {Matthew D. Zeiler and Rob Fergus},
  title     = {Visualizing and Understanding Convolutional Networks},
  booktitle = {European Conference on Computer Vision (ECCV)},
  pages     = {818--833},
  year      = {2014},
  doi       = {10.1007/978-3-319-10590-1_53}
}

@article{FedSD2C,
  title={One-shot federated learning via synthetic distiller-distillate communication},
  author={Zhang, Junyuan and Liu, Songhua and Wang, Xinchao},
  journal={Advances in Neural Information Processing Systems},
  volume={37},
  pages={102611--102633},
  year={2024}
}

@article{pflib,
  author  = {Jianqing Zhang and Yang Liu and Yang Hua and Hao Wang and Tao Song and Zhengui Xue and Ruhui Ma and Jian Cao},
  title   = {PFLlib: A Beginner-Friendly and Comprehensive Personalized Federated Learning Library and Benchmark},
  journal = {Journal of Machine Learning Research},
  year    = {2025},
  volume  = {26},
  number  = {50},
  pages   = {1--10},
  url     = {http://jmlr.org/papers/v26/23-1634.html}
}

@inproceedings{Pre-Text,
author = {Hou, Charlie and Shrivastava, Akshat and Zhan, Hongyuan and Conway, Rylan and Le, Trang and Sagar, Adithya and Fanti, Giulia and Lazar, Daniel},
title = {PrE-Text: training language models on private federated data in the age of LLMs},
year = {2024},
publisher = {JMLR.org},
booktitle = {Proceedings of the 41st International Conference on Machine Learning},
articleno = {766},
numpages = {19},
location = {Vienna, Austria},
series = {ICML'24}
}

@inproceedings{dp-prompt,
    title = "Locally Differentially Private Document Generation Using Zero Shot Prompting",
    author = "Utpala, Saiteja  and
      Hooker, Sara  and
      Chen, Pin-Yu",
    editor = "Bouamor, Houda  and
      Pino, Juan  and
      Bali, Kalika",
    booktitle = "Findings of the Association for Computational Linguistics: EMNLP 2023",
    month = dec,
    year = "2023",
    address = "Singapore",
    publisher = "Association for Computational Linguistics",
    url = "https://aclanthology.org/2023.findings-emnlp.566/",
    doi = "10.18653/v1/2023.findings-emnlp.566",
    pages = "8442--8457"}

@misc{dp-kde,
      title={Private Text Generation by Seeding Large Language Model Prompts}, 
      author={Supriya Nagesh and Justin Y. Chen and Nina Mishra and Tal Wagner},
      year={2025},
      eprint={2502.13193},
      archivePrefix={arXiv},
      primaryClass={cs.CL},
      url={https://arxiv.org/abs/2502.13193}, 
}

@INPROCEEDINGS{diffusion_federated,
  author={Hoefler, Maximilian Andreas and Mazouka, Tatsiana and Mueller, Karsten and Samek, Wojciech},
  booktitle={2024 IEEE International Conference on Big Data (BigData)}, 
  title={Boosting Federated Learning with Diffusion Models for Non-IID and Imbalanced Data}, 
  year={2024},
  volume={},
  number={},
  pages={7790-7799},
  doi={10.1109/BigData62323.2024.10825355}}

@article{dc-ae,
  title={Deep Compression Autoencoder for Efficient High-Resolution Diffusion Models},
  author={Chen, Junyu and Cai, Han and Chen, Junsong and Xie, Enze and Yang, Shang and Tang, Haotian and Li, Muyang and Lu, Yao and Han, Song},
  journal={arXiv preprint arXiv:2410.10733},
  year={2024}
}

@inproceedings{
balle,
title={Variational image compression with a scale hyperprior},
author={Johannes Ballé and David Minnen and Saurabh Singh and Sung Jin Hwang and Nick Johnston},
booktitle={International Conference on Learning Representations},
year={2018},
url={https://openreview.net/forum?id=rkcQFMZRb},
}

@article{dp-ntk,
  title={Differentially Private Neural Tangent Kernels (DP-NTK) for Privacy-Preserving Data Generation},
  author={Yang, Yilin and Adamczewski, Kamil and Li, Xiaoxiao and Sutherland, Danica J and Park, Mijung},
  journal={Journal of Artificial Intelligence Research},
  volume={81},
  pages={683--700},
  year={2024}
}

@inproceedings{dp-merf,
  title={Dp-merf: Differentially private mean embeddings with randomfeatures for practical privacy-preserving data generation},
  author={Harder, Frederik and Adamczewski, Kamil and Park, Mijung},
  booktitle={International conference on artificial intelligence and statistics},
  pages={1819--1827},
  year={2021},
  organization={PMLR}
}

@inproceedings{dpimagebench,
  title={Dpimagebench: A unified benchmark for differentially private image synthesis},
  author={Gong, Chen and Li, Kecen and Lin, Zinan and Wang, Tianhao},
  booktitle={Proceedings of the 2025 ACM SIGSAC Conference on Computer and Communications Security},
  pages={4139--4153},
  year={2025}
}

@inproceedings{dp-lora,
  title={Differentially private fine-tuning of diffusion models},
  author={Tsai, Yu-Lin and Li, Yizhe and Yu, Chia-Mu and Ren, Xuebin and Chen, Po-Yu and Chen, Zekai and Buet-Golfouse, Francois},
  booktitle={Proceedings of the IEEE/CVF International Conference on Computer Vision},
  pages={4561--4571},
  year={2025}
}

@inproceedings{dp-kernel,
author = {Jiang, Dihong and Sun, Sun and Yu, Yaoliang},
title = {Functional R\'{e}nyi differential privacy for generative modeling},
year = {2023},
publisher = {Curran Associates Inc.},
address = {Red Hook, NY, USA},
booktitle = {Proceedings of the 37th International Conference on Neural Information Processing Systems},
articleno = {652},
numpages = {21},
location = {New Orleans, LA, USA},
series = {NIPS '23}
}

@inproceedings{fedxds,
  title={FedXDS: Leveraging Model Attribution Methods to counteract Data Heterogeneity in Federated Learning},
  author={Hoefler, Maximilian Andreas and Mueller, Karsten and Samek, Wojciech},
  booktitle={Proceedings of the IEEE/CVF International Conference on Computer Vision},
  pages={4572--4581},
  year={2025}
}

@inproceedings{secagg,
author = {Bonawitz, Keith and Ivanov, Vladimir and Kreuter, Ben and Marcedone, Antonio and McMahan, H. Brendan and Patel, Sarvar and Ramage, Daniel and Segal, Aaron and Seth, Karn},
title = {Practical Secure Aggregation for Privacy-Preserving Machine Learning},
year = {2017},
isbn = {9781450349468},
publisher = {Association for Computing Machinery},
address = {New York, NY, USA},
url = {https://doi.org/10.1145/3133956.3133982},
doi = {10.1145/3133956.3133982},
booktitle = {Proceedings of the 2017 ACM SIGSAC Conference on Computer and Communications Security},
pages = {1175–1191},
numpages = {17},
location = {Dallas, Texas, USA},
series = {CCS '17}
}

@article{fedpft,
  title={Foundation models meet federated learning: A one-shot feature-sharing method with privacy and performance guarantees},
  author={Beitollahi, Mahdi and Bie, Alex and Hemati, Sobhan and Brunswic, Leo Maxime and Li, Xu and Chen, Xi and Zhang, Guojun},
  journal={Transactions on Machine Learning Research},
  year={2025}
}

@INPROCEEDINGS{MIA,
  author={Carlini, Nicholas and Chien, Steve and Nasr, Milad and Song, Shuang and Terzis, Andreas and Tramèr, Florian},
  booktitle={2022 IEEE Symposium on Security and Privacy (SP)}, 
  title={Membership Inference Attacks From First Principles}, 
  year={2022},
  volume={},
  number={},
  pages={1897-1914},
  doi={10.1109/SP46214.2022.9833649}}

@article{class_shuffling,
author = {Kairouz, Peter and McMahan, H. Brendan and Avent, Brendan and Bellet, Aur\'{e}lien and Bennis, Mehdi and Nitin Bhagoji, Arjun and Bonawitz, Kallista and Charles, Zachary and Cormode, Graham and Cummings, Rachel and D’Oliveira, Rafael G. L. and Eichner, Hubert and El Rouayheb, Salim and Evans, David and Gardner, Josh and Garrett, Zachary and Gasc\'{o}n, Adri\`{a} and Ghazi, Badih and Gibbons, Phillip B. and Gruteser, Marco and Harchaoui, Zaid and He, Chaoyang and He, Lie and Huo, Zhouyuan and Hutchinson, Ben and Hsu, Justin and Jaggi, Martin and Javidi, Tara and Joshi, Gauri and Khodak, Mikhail and Konecn\'{y}, Jakub and Korolova, Aleksandra and Koushanfar, Farinaz and Koyejo, Sanmi and Lepoint, Tancr\`{e}de and Liu, Yang and Mittal, Prateek and Mohri, Mehryar and Nock, Richard and \"{O}zg\"{u}r, Ayfer and Pagh, Rasmus and Qi, Hang and Ramage, Daniel and Raskar, Ramesh and Raykova, Mariana and Song, Dawn and Song, Weikang and Stich, Sebastian U. and Sun, Ziteng and Suresh, Ananda Theertha and Tram\`{e}r, Florian and Vepakomma, Praneeth and Wang, Jianyu and Xiong, Li and Xu, Zheng and Yang, Qiang and Yu, Felix X. and Yu, Han and Zhao, Sen},
title = {Advances and Open Problems in Federated Learning},
year = {2021},
issue_date = {Jun 2021},
publisher = {Now Publishers Inc.},
address = {Hanover, MA, USA},
volume = {14},
number = {1–2},
issn = {1935-8237},
url = {https://doi.org/10.1561/2200000083},
doi = {10.1561/2200000083},
journal = {Found. Trends Mach. Learn.},
month = jun,
pages = {1–210},
numpages = {214}
}
\bibliographystyle{tmlr}

\appendix

\section{Scaling with more Clients}
\label{sec:scaling}

Table~\ref{tab:scaling_1000} examines the effect of increasing the number of clients to 1000 under $\alpha{=}0.1$. All baselines degrade substantially as each client's local dataset shrinks and model-based one-shot methods must aggregate increasingly poor local models. FedSD2C, the most resilient baseline, drops from 76\% to 62\% on EuroSAT. FedKT-CSD, by contrast, is \textbf{invariant to federation size}: because the transmitted sums are additive, the global sums, and therefore the synthetic dataset, are identical regardless of how data is partitioned across clients.

\begin{table}[ht]
\centering
\caption{\textbf{Scaling to 1000 clients} ($\alpha{=}0.1$).
OSFL baselines degrade substantially; our method is invariant to client count.}
\label{tab:scaling_1000}
\begin{tabular}{l cccc}
\toprule
\textbf{Method} & \textbf{EuroSAT} & \textbf{CIFAR-100} & \textbf{ImageNette} & \textbf{BloodMNIST} \\
\midrule
FedSD2C            & 62.47 & 21.83 & 41.52 & 55.36 \\
DENSE              & 32.18 & 4.63  & 18.74 & 30.52 \\
CoBoosting         & 28.45 & 9.82  & 16.37 & 27.63 \\
FedD3              & 30.72 & 3.15  & 15.83 & 33.47 \\
FedCVAE            & 31.46 & 3.28  & 16.25 & 34.82 \\
\midrule
\textbf{Ours} ($\varepsilon{=}10$) & \textbf{81.09} & \textbf{28.47} & \textbf{60.95} & \textbf{75.71} \\
\bottomrule
\end{tabular}
\end{table}

\begin{table}[ht]
\centering
\caption{\textbf{One-shot FL comparison --- 100 clients.}
Same setup as the main paper. All baselines degrade with more clients;
our method is invariant to federation size.}
\label{tab:osfl_100c}
\resizebox{\textwidth}{!}{%
\begin{tabular}{ll ccc ccc ccc ccc}
\toprule
& & \multicolumn{3}{c}{\textbf{EuroSAT}}
  & \multicolumn{3}{c}{\textbf{CIFAR-100}}
  & \multicolumn{3}{c}{\textbf{ImageNette}}
  & \multicolumn{3}{c}{\textbf{BloodMNIST}} \\
\cmidrule(lr){3-5}\cmidrule(lr){6-8}\cmidrule(lr){9-11}\cmidrule(lr){12-14}
& \textbf{Method}
  & $\alpha{=}0.1$ & $\alpha{=}0.05$ & Path.
  & $\alpha{=}0.1$ & $\alpha{=}0.05$ & Path.
  & $\alpha{=}0.1$ & $\alpha{=}0.05$ & Path.
  & $\alpha{=}0.1$ & $\alpha{=}0.05$ & Path. \\
\midrule
& FedSD2C
  & 72.46 & 65.13 & 51.38
  & 27.19 & 19.15 & 25.94
  & 53.85 & 44.72 & 32.46
  & 65.83 & 57.46 & 62.15 \\
& DENSE
  & 50.27 & 41.83 & 32.46
  & 4.63 & 2.99 & 2.48
  & 33.58 & 24.73 & 16.85
  & 47.52 & 41.36 & 48.73 \\
& CoBoosting
  & 46.83 & 38.52 & 29.74
  & 12.80 & 11.42 & 9.41
  & 30.47 & 21.85 & 15.36
  & 43.67 & 38.42 & 45.82 \\
& FedD3
  & 35.62 & 31.47 & 25.83
  & 1.73 & 1.68 & 2.03
  & 19.35 & 17.42 & 14.28
  & 46.14 & 41.27 & 36.53 \\
& FedCVAE
  & 34.28 & 30.56 & 24.72
  & 1.65 & 1.70 & 2.12
  & 20.18 & 18.06 & 14.83
  & 47.62 & 42.38 & 37.15 \\
\midrule
& \textbf{Ours (DP)}
  & \textbf{81.52} & \textbf{82.07} & \textbf{82.11}
  & \textbf{28.62} & \textbf{28.99} & \textbf{30.20}
  & \textbf{62.15} & \textbf{62.28} & \textbf{61.37}
  & \textbf{75.44} & \textbf{75.59} & \textbf{77.04} \\
\bottomrule
\end{tabular}
}
\end{table}

\section{Comparison of Output Modality: Pixel-Space vs.\ Feature-Space Synthesis}
\label{app:output_modality}

A fundamental distinction between FedKT-CSD and other methods such as FedPFT~\cite{fedpft}
or FedSD2C~\cite{FedSD2C} lies in the modality of the synthesized output. FedPFT generates
synthetic features in the embedding space of a specific frozen foundation model (e.g.\ CLIP),
and FedSD2C generates only small distilled datasets. By contrast, FedKT-CSD decodes its
DP-protected latent statistics back to pixel space via the frozen decoder, producing a
synthetic image dataset $X^*$ that is fully model-agnostic.

\paragraph{Implications for downstream use.}
The pixel-space output of FedKT-CSD unlocks a strictly broader set of downstream
applications than feature-space synthesis:

\begin{itemize}
    \item \textbf{Architecture flexibility.} Any model architecture can be trained on
    $X^*$ directly, including convolutional networks, vision transformers, or
    task-specific models, without any dependency on the autoencoder used to generate it. We demonstrated this on a p-FL downstream task.
    
    \item \textbf{Pretraining and transfer learning.} $X^*$ can serve as a pretraining
    corpus for any feature extractor, enabling transfer to downstream tasks beyond the
    original classification problem.

    \item \textbf{Knowledge distillation.} $X^*$ can be used as a distillation corpus
    for any teacher-student pair, independent of the teacher's architecture.

    \item \textbf{Data augmentation.} $X^*$ can be combined with local client data for
    fine-tuning or mixed into any training pipeline as auxiliary data.

    \item \textbf{Multi-task reuse.} Because $X^*$ consists of images rather than
    task-specific features, it can be reused across multiple downstream tasks without
    additional privacy cost --- generating more samples from the DP-protected Gaussians
    incurs no further privacy expenditure by the post-processing theorem.

    \item \textbf{Third-party sharing.} $X^*$ can be shared with parties who have no
    knowledge of the autoencoder architecture or weights, since it is simply a labeled
    image dataset.
\end{itemize}

\section{Detailed Comparison with FedPFT}
\label{app:fedpft}

FedPFT~\cite{fedpft} is the most closely related prior work: both methods encode images through a frozen pretrained model, compute class-conditional distributions, and transmit them in one round. FedPFT includes a theoretical treatment of DP and some privacy evaluation, but it is not a core architectural design choice. As we detail below, the way noise is applied in FedPFT is likely to degrade rapidly at meaningful privacy budgets. The comparison below reflects our interpretation of their described mechanism; we encourage readers to consult the original work.

The most consequential difference is noise scaling. In FedPFT, each client adds noise to its own GMM parameters before uploading. With $N$ clients each holding $n/N$ samples, each client's noise is calibrated to its local sample size. For $K{>}1$ components the server concatenates noisy GMMs without averaging, so each synthetic feature carries one client's full noise. For $K{=}1$ the server can average, reducing noise by $\sqrt{N}$, but without secure aggregation the server still observes individual contributions. In our approach, clients upload masked additive sums via secure aggregation; the server sees only the global total, and the effective noise matches the centralized setting with all $n$ samples. With 100 clients this yields roughly $10\times$ better SNR than FedPFT $K{=}1$.

Second, FedPFT assumes $\|f\|_2 \leq 1$, which holds for CLIP but not for general extractors. Our autoencoder prior provides a known norm distribution from which we derive a data-independent clipping radius, making the guarantee unconditional on model choice. Third, FedPFT does not account for composition across classes, and their $K{>}1$ GMM fitting involves sequential steps (K-means then per-component estimation) that require sequential composition. Our single-pass additive sums avoid this complexity: for each class we release only two aggregated statistics in the main text (or three with privatized counts), so the privacy accounting is simple. Finally, our decoder produces reusable pixel-space images rather than features tied to a specific classifier head.

\section{Parallel Composition and Privacy Accounting}
\label{app:parallel_composition}

A key structural property of our pipeline is that, under the public-label assumption,
each record contributes to exactly one class-specific mechanism. This enables parallel
composition across classes, while the multiple releases within a class compose
sequentially.

\begin{theorem}[Parallel Composition~\cite{dwork2014algorithmic}]
If $D = D_1 \cup \dots \cup D_k$ are disjoint subsets and each mechanism $M_i$
applied to $D_i$ satisfies $(\varepsilon_i,\delta_i)$-DP, then the joint release
of all outputs satisfies $(\max_i \varepsilon_i,\max_i \delta_i)$-DP.
\end{theorem}

In our setting, the dataset decomposes across classes as
$D = D_{c=1} \cup \dots \cup D_{c=K}$, since each sample has exactly one class label.
For a fixed class $c$, the mechanism releases two noisy quantities in the main text:
the mean sum $\widetilde{\mathbf S}_{\mu,c}$ and the second-moment sum
$\widetilde{\mathbf S}_{\Sigma,c}$. These two releases access the same class-$c$
records and therefore compose sequentially within the class. In the add/remove
extension of Section~\ref{app:addremove}, the privatized count $\tilde n_c$ adds a
third sequentially composed release.

Across classes, however, the class-specific datasets are disjoint, so the classwise
mechanisms compose in parallel. Therefore, under within-class replacement adjacency
and public labels, the overall privacy cost is the privacy cost of a single
class-specific mechanism: $(\varepsilon,\delta)$ in the main text, and
$(\varepsilon,\delta)$ under the budget split used in
Section~\ref{app:addremove}.

Secure aggregation does not change the DP accounting itself; it only hides individual
client contributions before the global classwise sums are formed.

\section{Further Visualization and Justification of Image Quality}

\begin{figure*}[t]
\centering
\includegraphics[width=0.95\textwidth]{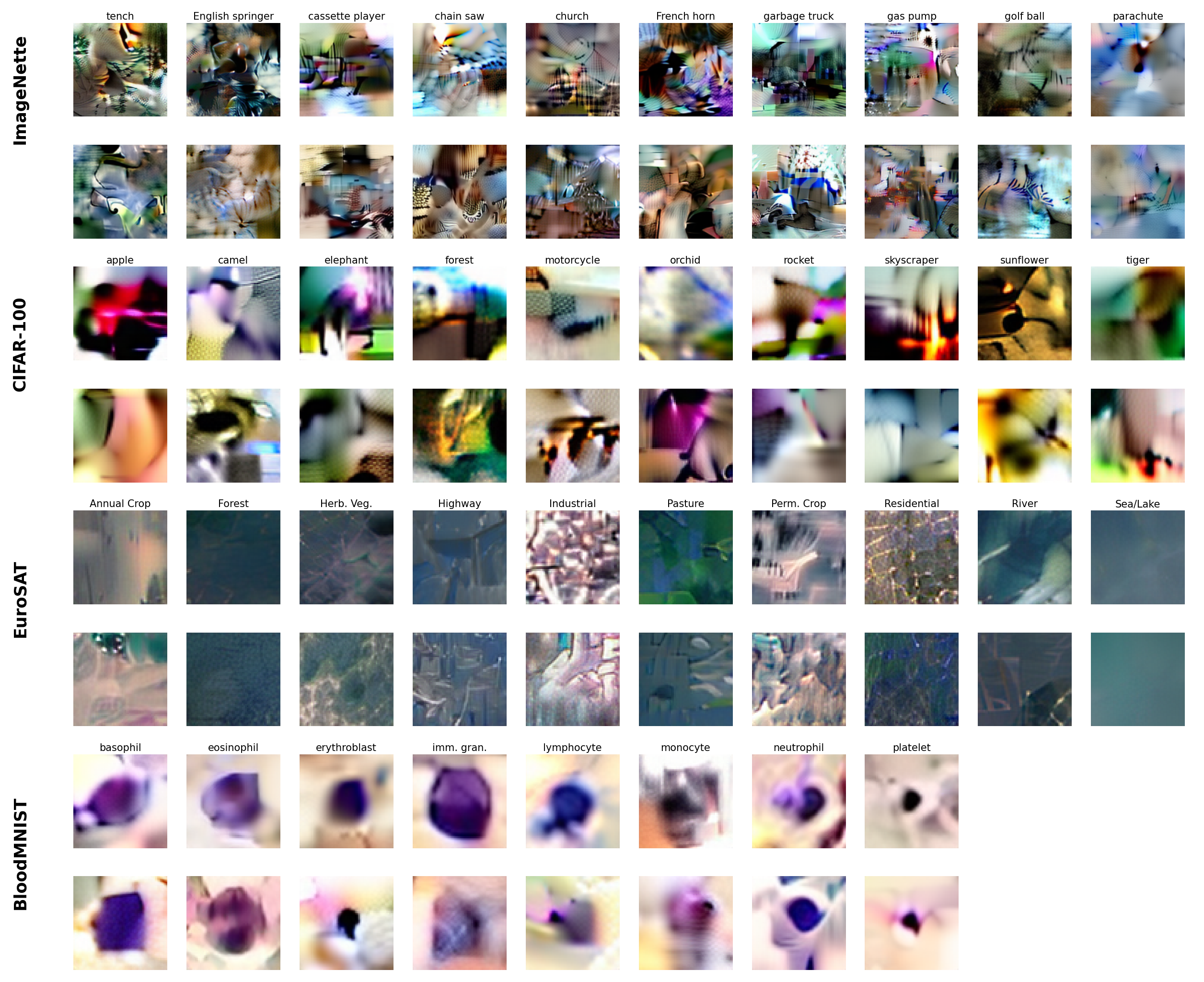}
\caption{Synthetic images generated by FedKT-CSD ($\varepsilon{=}10$, $\delta{=}10^{-5}$) for each dataset.}
\label{fig:all_samples}
\end{figure*}

We show generated images of our method (\autoref{fig:all_samples}) for all datasets used, using at most 10 classes from each dataset for better visualization. For EuroSAT and BloodMNIST the images are clearly recognizable with respect to the original datasets and their respective classes. However, for ImageNette and CIFAR-100 the generated images are more abstract representations. There are several reasons why these representations remain useful for downstream classification. Prior work has established that deep networks do not require human-recognizable image quality to learn effective representations. \citet{Geirhos2019TextureBias} showed that convolutional networks rely heavily on local texture statistics rather than global shape, achieving high accuracy even when shape information is disrupted. \citet{Nguyen2015EasilyFooled} demonstrated that patterns appearing as noise to humans can elicit near-perfect classifier confidence, confirming that networks latch onto features orthogonal to human visual criteria. On the resolution front, \citet{Wu2019LowResTransfer} showed that embeddings learned from images as small as $16\times16$ pixels transfer competitively to standard benchmarks. More broadly, \citet{Yosinski2014Transferable} found that mid-to-upper convolutional features remain highly reusable across domains regardless of source image appearance, and \citet{Zeiler2014Visualizing} visualized intermediate activations that, while lacking human interpretability, are instrumental for classification. Our synthetic images preserve exactly the kind of local texture and class-conditional structure that these studies identify as sufficient for representation learning.

\section{Extension to Add/Remove Adjacency via Count Privatization}
\label{app:addremove}

The main text adopts within-class replacement adjacency, which fixes the per-class count
$n_c$ and protects the \emph{values} of individual records but not participation itself.
Here we show that a straightforward extension recovers full add/remove record-level DP.

\paragraph{Add/remove adjacency.}
Under add/remove adjacency, adjacent datasets $\mathcal{D} \sim \mathcal{D}'$ differ by the
addition or removal of a single class-$c$ record. This changes three quantities:
$\mathbf{S}_{\mu,c}$ by at most $R$ in $\ell_2$-norm, $\mathbf{S}_{\Sigma,c}$ by at most
$R^2$ in Frobenius norm, and $n_c$ by exactly $1$. The sensitivities are therefore:
\begin{equation}
\Delta_\mu^{\rm ao} = R, \qquad \Delta_\Sigma^{\rm ao} = R^2, \qquad \Delta_n = 1.
\end{equation}
Note that the sensitivities on $\mathbf{S}_{\mu,c}$ and $\mathbf{S}_{\Sigma,c}$ are halved
compared to replacement adjacency in the main text, since only one term is added
or removed rather than swapped.

\paragraph{Budget allocation.}
We partition the total privacy budget $(\varepsilon, \delta)$ as
$\varepsilon = \varepsilon_{\rm stats} + \varepsilon_n$ and
$\delta = \delta_{\rm stats} + \delta_n$, allocating
$(\varepsilon_{\rm stats},\delta_{\rm stats})$ to the mean and covariance sums and
$(\varepsilon_n,\delta_n)$ to the count. Splitting
$(\varepsilon_{\rm stats},\delta_{\rm stats})$ equally between the two statistics, we
set $\sigma_\mu^{\rm ao}$, $\sigma_\Sigma^{\rm ao}$, and $\sigma_n$ to be the smallest
positive values satisfying
\begin{equation}
\Phi\!\left(\frac{\Delta}{2\sigma} - \frac{\varepsilon' \sigma}{\Delta}\right)
-
e^{\varepsilon'}
\Phi\!\left(-\frac{\Delta}{2\sigma} - \frac{\varepsilon' \sigma}{\Delta}\right)
\le \delta',
\end{equation}
with $(\varepsilon',\delta') =
(\varepsilon_{\rm stats}/2,\delta_{\rm stats}/2)$ and $\Delta = R$ for
$\sigma_\mu^{\rm ao}$, $(\varepsilon',\delta') =
(\varepsilon_{\rm stats}/2,\delta_{\rm stats}/2)$ and $\Delta = R^2$ for
$\sigma_\Sigma^{\rm ao}$, and $(\varepsilon',\delta') =
(\varepsilon_n,\delta_n)$ and $\Delta = 1$ for $\sigma_n$.
The noisy count $\tilde n_c = n_c + \eta_n$ with
$\eta_n \sim \mathcal{N}(0,\sigma_n^2)$ is then used in place of $n_c$
throughout the post-processing steps in the main text.

\begin{proposition}[Add/Remove Record-Level DP]
\label{prop:addremove_dp}
Under the notation above, the joint release of
$(\widetilde{\mathbf{S}}_{\mu,c},\, \widetilde{\mathbf{S}}_{\Sigma,c},\, \tilde{n}_c)$
satisfies $(\varepsilon, \delta)$-differential privacy under add/remove adjacency.
\end{proposition}

\begin{proof}
By the analytic Gaussian mechanism~\cite{analytic_gaussian}, each query satisfies DP
under its allocated sub-budget when Gaussian noise is calibrated according to the
displayed condition above. Thus
$\widetilde{\mathbf S}_{\mu,c}$ is
$(\varepsilon_{\rm stats}/2,\delta_{\rm stats}/2)$-DP,
$\widetilde{\mathbf S}_{\Sigma,c}$ is
$(\varepsilon_{\rm stats}/2,\delta_{\rm stats}/2)$-DP, and
$\tilde n_c$ is $(\varepsilon_n,\delta_n)$-DP.
By basic composition, the joint release satisfies
$(\varepsilon_{\rm stats}+\varepsilon_n,\delta_{\rm stats}+\delta_n)
= (\varepsilon,\delta)$-DP.
All subsequent operations are deterministic post-processing and preserve DP.
\end{proof}

We note that compared to the replacement adjacency in the main text, the noise on
$\mathbf{S}_{\mu,c}$ and $\mathbf{S}_{\Sigma,c}$ is reduced by a factor of two due to the
halved sensitivity, partially offsetting the additional budget allocated to the count. The
utility impact of count privatization depends on the ratio $\sigma_n / n_c$ and is left
for future empirical investigation.

\section{Secure Aggregation Details}
\label{app:secagg}

We provide additional detail on how the secure aggregation protocol of \cite{secagg} applies to our setting. We assume an honest-but-curious server: the server follows the protocol correctly but attempts to infer as much as possible from the messages it receives.

\paragraph{Privacy scope.}
The differential privacy guarantee in the main text applies to the released noisy
statistics (and any downstream synthetic data derived from them). In the protocol
described here, secure aggregation hides individual client contributions, but the
server recovers the exact global sums before adding DP noise. Accordingly, our threat
model assumes that the server is trusted to execute the DP mechanism on the aggregated
statistics after secure aggregation. An alternative is distributed noise addition, in
which no raw aggregate is ever revealed; we leave this variant for future work.

\paragraph{Setting.}
Each client $i \in \{1,\dots,N\}$ holds a private vector $\mathbf{m}_c^i \in \mathbb{R}^d$ (the class-$c$ latent sum) and a private matrix $\mathbf{M}_c^i \in \mathbb{R}^{d \times d}$ (the class-$c$ outer-product sum). The server requires only the element-wise totals $\mathbf{S}_{\mu,c} = \sum_i \mathbf{m}_c^i$ and $\mathbf{S}_{\Sigma,c} = \sum_i \mathbf{M}_c^i$, and should learn nothing about any individual client's contribution beyond what is inferable from these totals. We describe the protocol for $\mathbf{m}_c^i$; the same procedure is applied independently to each entry of the symmetric upper triangle of $\mathbf{M}_c^i$ and to $n_c^i$.

\paragraph{Masking with pairwise perturbations.}
The basic idea, following~\cite{secagg}, is to mask each client's input with random perturbations that cancel upon summation. Each pair of clients $(i,j)$ agrees on a shared random seed $s_{i,j} = s_{j,i}$ via Diffie-Hellman key agreement, mediated by the server. From this seed, both clients derive a pseudorandom vector $\mathbf{p}_{i,j} \in \mathbb{R}^d$ using a cryptographically secure pseudorandom generator. Client $i$ computes a perturbation by summing over all other clients:
\begin{equation}
\mathbf{r}_i = \sum_{\substack{j \neq i \\ j > i}} \mathbf{p}_{i,j} - \sum_{\substack{j \neq i \\ j < i}} \mathbf{p}_{i,j},
\end{equation}
so that by construction $\sum_{i=1}^N \mathbf{r}_i = \mathbf{0}$, since each pairwise term $\mathbf{p}_{i,j}$ appears once with a positive sign (for the client with smaller index) and once with a negative sign (for the client with larger index). Each client uploads the masked value $\mathbf{y}_i = \mathbf{m}_c^i + \mathbf{r}_i$ to the server, which computes:
\begin{equation}
\sum_{i=1}^N \mathbf{y}_i = \sum_{i=1}^N \mathbf{m}_c^i + \sum_{i=1}^N \mathbf{r}_i = \sum_{i=1}^N \mathbf{m}_c^i = \mathbf{S}_{\mu,c}.
\end{equation}
The server recovers the exact global sum. Under the assumption that the pseudorandom generator is cryptographically secure, each individual $\mathbf{y}_i$ is computationally indistinguishable from uniform randomness to the server, revealing nothing about $\mathbf{m}_c^i$ beyond what is inferable from the global sum $\mathbf{S}_{\mu,c}$.

\paragraph{Dropout resilience.}
If a client $i$ drops out before uploading $\mathbf{y}_i$, its perturbations would not cancel, corrupting the sum. To handle this, each client secret-shares its private key (used to derive the pairwise seeds) among all other clients using a $(t,N)$-threshold scheme such as Shamir's Secret Sharing, where $t > N/2$. If client $i$ drops out, the surviving clients contribute their shares of $i$'s key, allowing the server to reconstruct $i$'s perturbations and subtract them from the running total. This ensures correctness as long as at least $t$ clients complete the protocol. However, if the server were to reconstruct a surviving client's pairwise perturbations, it could potentially unmask that client's input --- this is addressed by double masking.

\paragraph{Double masking.}
Each client $i$ adds a second independent mask $\mathbf{b}_i$ (also secret-shared among other clients) to its upload: $\mathbf{y}_i = \mathbf{m}_c^i + \mathbf{b}_i + \mathbf{r}_i$. During unmasking, the server must choose for each client: either recover the pairwise perturbations $\mathbf{r}_i$ (for genuine dropouts) or recover $\mathbf{b}_i$ (for surviving clients). An honest client will never reveal both types of shares for the same user, ensuring that each client's input remains masked by at least one layer at all times, even if the server falsely claims a client dropped out.

\paragraph{Why this fits our method.}
Standard SecAgg was designed for aggregating model gradients in federated learning, where each client's contribution is a high-dimensional vector and the server needs only the sum. Our setting is a natural fit for the same reason: the per-class statistics $(\mathbf{m}_c^i, \mathbf{M}_c^i, n_c^i)$ are precisely the kind of additive quantities that SecAgg is designed for, and the server never needs to inspect individual contributions. Exploiting the symmetry of $\mathbf{M}_c^i$, each client's payload consists of $K \times (d + d(d{+}1)/2 + 1)$ floats. With $d{=}128$ and $K{=}10$ classes this amounts to approximately 82k floats per client, or 327.5\,KB at 4 bytes per float --- roughly three orders of magnitude smaller than a typical ResNet-18 gradient upload (${\sim}45$\,MB). This makes the pairwise key agreement and masking steps lightweight in both communication and computation.

\end{document}